\documentclass[twoside]{article}

\usepackage[accepted]{aistats2026}

\usepackage[utf8]{inputenc}
\usepackage{amsmath,amsfonts}
\usepackage{amssymb}
\usepackage{graphicx}
\usepackage{tikz}
\usepackage{microtype}
\usepackage{subfigure}
\usepackage{natbib}
\usepackage{mathtools}
\newcommand{\R}{\mathbb{R}}
\usepackage{amsthm}
\usepackage{nccmath}
\usepackage{xspace}
\usepackage{xcolor}
\usepackage{booktabs}
\usepackage{array}
\usepackage{subcaption}
\definecolor{MyBlue}{rgb}{0.12, 0.12, 0.76}
\usepackage[colorlinks,allcolors=MyBlue]{hyperref}
\usepackage{pifont}% http://ctan.org/pkg/pifont

\usepackage[group-separator={,}]{siunitx}
\usepackage{url}
\usepackage{enumitem}
\usepackage{pdflscape}
\usepackage{verbatim}
\usepackage{arydshln}
\usepackage{multirow}
\usepackage{bigdelim}
\usepackage{stfloats}

\usepackage{pgfplots}
\usepackage{cleveref}

\usetikzlibrary{shadows,arrows,decorations,decorations.shapes,backgrounds,shapes,snakes,automata,fit,petri,shapes.multipart,calc,positioning,shapes.geometric,graphs,graphs.standard,plotmarks,math,arrows.meta}
\usepackage{tikz-qtree}

\usepackage{thmtools}
\usepackage{thm-restate}

\usepackage{algorithm}
\usepackage{algpseudocode}
\usepackage{algorithmicx}
\let\oldReturn\Return
\renewcommand{\Return}{\state\oldReturn}
\algtext*{EndWhile}% Remove "end while" text
\algtext*{EndIf}
\algtext*{EndForAll}
\algtext*{EndFor}
\algtext*{EndFunction}

\DeclareMathOperator*{\E}{\mathbb{E}}

\allowdisplaybreaks % allow align to break over multiple pages

\newcommand\bbr{\mathbb{R}}
\newcommand\bbrpos{\mathbb{R}_{\geq 0}}
\newcommand\bbrspos{\mathbb{R}_{> 0}}
\newcommand\bbn{\mathbb{N}}
\newcommand\ep{\varepsilon}

\newcommand\bfone{\boldsymbol{1}}
\newcommand\norm[1]{\|#1\|}
\newcommand\X{\mathcal{X}}

\newcommand\hmub{\hat{\mu}_B}

\newcommand{\G}{\mathcal{G}}

\newcommand{\B}{\mathcal{B}}

\def\B{\mathcal{B}}

\def\B{\mathcal{B}}

\def\G{\mathcal{G}}

\def\R{\mathbb{R}}

\def\RegT{\textnormal{Reg}(T)}

\newcommand\D{\mathcal{D}}

\newenvironment{proofidea}[1][]%
{\par\noindent\textit{Proof idea#1.}\ }%
{\hfill$\square$\par}

\declaretheorem[name=Lemma,sibling=theorem]{lemma}

\crefname{theorem}{Thm.}{Thms.}
\Crefname{lemma}{Lem.}{Lems.}

\usepackage{tikz,xcolor}
\usetikzlibrary{patterns,arrows}
\renewcommand{\thefootnote}{\fnsymbol{footnote}}

\begin{document}

\runningtitle{Risk-Sensitive Abstention in Bandits}
\runningauthor{Sarah Liaw, Benjamin Plaut}
\twocolumn[

\aistatstitle{Learning When Not to Learn:\\ Risk-Sensitive Abstention in Bandits with Unbounded Rewards}
\aistatsauthor{
    Sarah Liaw\footnotemark[1]
    \And 
    Benjamin Plaut\footnotemark[1]
}
\aistatsaddress{Harvard University\\ University of California, Berkeley\And University of California, Berkeley}

]

\footnotetext[1]{Equal contribution.}
\renewcommand{\thefootnote}{\arabic{footnote}}

\begin{abstract}
    In high-stakes AI applications, even a single action can cause irreparable damage. However, nearly all of sequential decision-making theory assumes that all errors are recoverable (e.g., by bounding rewards). Standard bandit algorithms that explore aggressively may cause irreparable damage when this assumption fails. Some prior work avoids irreparable errors by asking for help from a mentor, but a mentor may not always be available. In this work, we formalize a model of learning with unbounded rewards without a mentor as a two-action contextual bandit with an abstain option: at each round the agent observes an input and chooses either to abstain (always 0 reward) or to commit (execute a preexisting task policy). Committing yields rewards that are upper-bounded but can be arbitrarily negative, and the commit reward is assumed Lipschitz in the input. We propose a caution-based algorithm that learns when not to learn: it chooses a trusted region and commits only where the available evidence does not already certify harm. Under these conditions and i.i.d. inputs, we establish sublinear regret guarantees, theoretically demonstrating the effectiveness of cautious exploration for deploying learning agents safely in high-stakes environments.\looseness-1

\end{abstract}

\section{INTRODUCTION}

With artificial intelligence becoming ubiquitous, many learning systems are now deployed in unpredictable, safety-critical domains, such as process control and manufacturing robotics, autonomous driving, and surgical assistance. In these settings, a single ill-chosen action can cause irreparable and lasting damage with no opportunity for subsequent recovery. For instance, a self-driving car cannot compensate for a deadly crash by later driving more safely, nor can a medical robot undo a fatal mistake during surgery. Following \citet{plaut2025safe,plaut2025avoiding}, we refer to such irreparable errors as \emph{catastrophes}. \looseness-1

Despite the risks such deployments pose, there is limited work (and limited theoretical work in particular) on how an agent can learn without ever incurring an irreparable error. The possibility of catastrophes challenges standard frameworks for sequential decision making, especially the familiar notion of \emph{optimism under uncertainty}. Optimism effectively assumes that early mistakes can be offset (or be compensated for) by later gains, an assumption that is inappropriate when errors are irrecoverable. Instead, these settings call for \emph{pessimism under uncertainty}: when evidence is insufficient, prefer inaction to risky action. \looseness-1

One approach to mitigate these problems is to let the agent ask for help from a mentor in unfamiliar or risky situations. Such human-in-the-loop oversight can block unsafe actions and prevent irreparable errors (even if ordinary, recoverable errors still occur). However, this approach depends on the availability of a capable mentor, which can be costly or impractical at scale. This motivates a mentor-free alternative: \textbf{can an agent avoid irreparable errors on its own by acting cautiously when inputs appear unfamiliar?} \looseness-1

We propose a model of learning in the presence of irreparable costs without a mentor but with an option to abstain from action. The key question is when to abstain, i.e., when not to learn. To focus on this question, we assume the agent has previously learned a baseline policy that works well in-distribution but behaves unpredictably elsewhere. This allows us to streamline the model to two actions: \emph{abstain} (do nothing) and \emph{commit} (follow the baseline policy). Abstaining yields a deterministic safe reward $r(x,0)=0$, while committing yields a reward $r(x,1)\in(-\infty,1]$\footnote{The asymmetric bounds on the commit reward reflect that a single action can be catastrophic, whereas it is rare for a single action to yield arbitrarily large benefit.}. \looseness-1

We treat the origin as fully ``in-distribution'' and assume the baseline policy is beneficial there: $r(\mathbf{0},1)>0$. We use the distance from the origin $\|x\|$ as a measure of how out-of-distribution (OOD) an input is. The commit reward is assumed $L$-Lipschitz, capturing the idea that similar inputs yield similar outcomes. For our main results we focus on a fixed distribution $\nu$; for our impossibility results we also consider a $T$-dependent distribution $\nu_T$. \looseness-1

We formalize the tension between exploration and safety via two negative results. First, in the worst case, any algorithm that begins by always exploring (i.e., commits on the first round regardless of the input) can suffer infinite expected regret (\Cref{thm:neg-incautious}). Second, when every input lies uniformly far OOD, there is no safe way to explore to identify a beneficial committing region, and sublinear regret is impossible (\Cref{thm:neg-adv}). Together, these results delineate both the necessity and the limits of caution. \looseness-1

Motivated by this perspective, we develop a caution-based algorithm that learns only when it can guarantee that an error is not catastrophic (which essentially corresponds to not-too-OOD inputs). This approach yields sublinear expected regret for i.i.d. inputs from any fixed distribution, with bounds that also reflect how often the agent encounters far OOD inputs, while prioritizing the avoidance of irreparable errors. As in standard Lipschitz contextual bandits, however, this regret bound has an exponential dependence on the dimension $n$, and is therefore most meaningful in low-dimensional settings.\looseness-1

\textbf{Contributions.} Our contributions are:
\begin{enumerate}[leftmargin=1.5em,
    topsep=0.5ex,
    partopsep=0pt,
    parsep=0pt,
    itemsep=0.5ex]
  \item We introduce a formal model of learning with irreparable costs and no external mentor.
  
  \item We prove two impossibility results that delineate the necessity and limits of caution.
  
  \item We develop a caution-based algorithm that achieves sublinear regret for any fixed input distribution.
\end{enumerate}

\textbf{Organization.} \autoref{sec:model} introduces the formal model and notation. \autoref{sec:virtues-limits} presents the impossibility results (\Cref{thm:neg-incautious,thm:neg-adv}) and their implications for exploration. \autoref{sec:algorithm} describes our caution-based learning algorithm and states the main regret bound. \autoref{sec:roadmap} outlines the proof strategy and supporting lemmas. \autoref{app:proofs} presents the complete proof, and \autoref{app:sim} provides some numerical simulations. \looseness-1

\section{RELATED WORK}\label{sec:related}
Most prior work on sequential decision-making and safe exploration focuses on settings where errors are ultimately recoverable; here we contrast this with our setting where individual actions can cause irreparable harm. \looseness-1

\subsection{Sequential decision-making when all errors are recoverable}
The literature on sequential decision-making is vast, spanning bandit problems, reinforcement learning, and online learning. See \citet{slivkins2019introduction}, \citet{sutton1998reinforcement}, and \citet{cesa2006prediction} for introductions to these (somewhat overlapping) topics, respectively. However, nearly all of this work assumes explicitly or implicitly that any error can be recovered from. This assumption enables the agent to ignore risk and simply try all possible behaviors, since no matter how badly it performs in the short term, it can always eventually make up for it. Indeed, most sequential decision-making algorithms with formal regret bounds have this general structure. \looseness-1

This assumption can manifest in different ways. In bandit settings, it suffices to assume that rewards are bounded (or at least have bounded expectation). This assumption implies that the expected regret from any action on any time step is always bounded, which is sufficient for the risk-agnostic exploration mentioned above. In contrast, we allow unbounded negative rewards so that actions can be arbitrarily costly. Indeed, our first negative result (\Cref{thm:neg-incautious}) relies on the expected regret for a single action potentially being infinite in our model. 

In Markov Decision Processes (MDPs), the agent's actions determine the next state via a transition function, so in addition to bounded rewards, one typically assumes that either the environment is reset at each of each ``episode'' (e.g., \citealp{azar_minimax_2017}) or that any state is reachable from any other (e.g., \citealp{jaksch_near-optimal_2010}). The reliance of standard MDP algorithms on these assumptions was observed by \citet{moldovan_safe_2012, cohen_curiosity_2021}, among others. 

Regardless of the specific form of this assumption, it clearly does not hold in safety-critical contexts where a single action can be catastrophic. 

\subsection{Safe exploration}

These issues have motivated a wide field of safe exploration. A full survey is beyond the scope of this paper (see \citealp{garcia_comprehensive_2015, gu_review_2024, krasowski_provably_2023,tan2022survey} for surveys), so we cover only the most relevant prior work. Avoiding irreparable errors while learning has also been studied empirically across multiple domains (e.g., \citealp{saunders2017trialerrorsafereinforcement, moldovan2012safeexplorationmarkovdecision, wachi2023safeexplorationreinforcementlearning,zhao_state-wise_2023,perkins-lyapunov-safe-rl2003}), but here we focus on theoretical work, which is most relevant to our setting. 

Safe exploration is modeled in two main ways. The first approach is to require the agent to satisfy some sort of constraint in addition to maximizing reward. The constraint can be entirely separate from reward, as in the case of constrained MDPs \citep{altman1999constrained}, or they can be related to the reward (e.g., the agent's reward must always exceed some baseline). When zero or near-zero constraint violation is required, these formalisms do capture the possibility of irreparable errors. The second approach treats reward as the sole objective, with safety as a necessary but not sufficient property for maximizing reward. Here, irreparable errors correspond to either unboundedly negative rewards (our work falls into this category) or inescapable ``trap'' states with poor reward. An agent that obtains very negative rewards or enters trap states clearly cannot obtain high reward. 

Both of these models must contend with a fundamental obstacle: how does one learn which actions are catastrophic without trying those actions directly? This can be formalized by the so called ``Heaven or Hell problem''~\citep{plaut2025safe}. Suppose there are two available actions, where one has unbounded positive reward and the other has unbounded negative reward. In this case, the agent can do no better than simply guessing and can never guarantee good regret. This problem shows that some sort of additional assumption is necessary for any meaningful regret guarantees. Below, we categorize work within safe exploration based on which assumption(s) it uses for this purpose. 

\textbf{Full prior knowledge.} Perhaps the simplest approach is to assume that the agent knows the precise safety constraint upfront (see \citealp{zhao_state-wise_2023} for a survey). This immediately resolves the Heaven or Hell problem; indeed, it eliminates the need for the agent to ``learn when not to learn'' at all. However, full knowledge of the safety constraint may not hold in practice. In contrast, we only assume that the (1) baseline policy performs well in-distribution and (2) the agent can always safely abstain. 

\textbf{Learning constraints using a safe fallback action.} There is a growing body of work which shares our assumption of a safe fallback action. \citet{liu2021learning, stradi2024learning} use this approach in the constrained MDP model, while \citet{wu2016conservativebandits,kazerouni2017conservativecontextuallinearbandits,lin2022stochasticconservativecontextuallinear,chen2022strategiessafemultiarmedbandits} require the reward to exceed a fixed baseline in a bandit model. These papers generally rely on a pair of subtle but crucial assumptions to obtain zero constraint violation: (1) the constraint violation on any given time step is bounded and (2) the baseline policy satisfies the constraints with a known amount of slack (this property is often referred to as \emph{Slater's gap}~\citep{bernasconi2024primaldualmethodsbanditsstochastic}, although not all of the above papers use this terminology). This combination of assumptions enables the agent to still explore aggressively with some known probability. Furthermore, the resulting bounds typically depend inversely on Slater's gap. 

Our work is complementary to each of these two assumptions. First, rather than assuming global boundedness, we assume that rewards decrease at a bounded \emph{rate}, i.e., rewards are Lipschitz continuous. Second, rather than dependence on the reward or cost function (in the form of Slater's gap), our bounds depend on the input distribution: specifically, our bounds degrade as the agent sees more OOD inputs. Our approach may be more or less realistic depending on the specific context, but it notably differs from the typical way fallback actions are utilized.

\textbf{Risk-sensitive contextual Lipschitz bandits.} We are not aware of prior work on risk-sensitive contextual bandits with Lipschitz rewards. To our knowledge, the Lipschitz contextual bandit literature (e.g., \citealp{slivkins_contextual_2011}) does not consider risk-sensitivity, and the risk-sensitive bandit literature is either non-contextual (e.g., \citealp{wu2016conservativebandits}) or studies linear bandits instead of Lipschitz bandits (e.g., \citealp{lin2022stochasticconservativecontextuallinear}). Because of these differences, it would be difficult to make an apples-to-apples comparison to these algorithms in our simulations, so we do not include these algorithms.

\textbf{Asking for help.} Perhaps the most common approach in this model is relying on external supervision. This is a growing body of work which uses limited queries to a mentor to prove formal regret guarantees in the presence of irreversible dynamics \citep{cohen_curiosity_2021, cohen_pessimism_2020, kosoy_delegative_2019, maillard_active_2019, plaut2025avoiding, plaut2025safe}. However, as the number of deployed AI systems continues to grow, it may be impractical for each one to have a human supervisor. Even in cases where external help will eventually become available, the agent may need to behave safely on its own in the short-term. These considerations motivate our study of how to learn safely in the absence of external help. 

\subsection{Other related work}

We briefly discuss some topics that are less directly relevant but still worth mentioning. One is the heavy-tailed bandit model \citep{bubeck2013bandits,agrawal2021regret}, which studies the case where reward distributions are not subgaussian and thus less predictable. While this model does incorporate elements of safety, as long as the expected reward from any action is bounded, risk-agnostic exploration remains valid (as discussed above). Another topic adjacent to our work is the standard Lipschitz bandit model with bounded rewards and bounded domain (see, e.g., Chapters 4 and 8 of \citealp{slivkins_contextual_2011}). This work shares some similarities with ours, like the algorithmic use of discretization. However, the core of our paper is removing the boundedness assumptions, which introduces a host of new challenges. Finally, there is complementary work on abstention with bounded rewards~\citep{neu2020fastratesonlineprediction,yang2024multiarmedbanditsabstention}. While this line of work also demonstrates the benefits of abstention, it does not address the possibility of irreparable errors.

\section{PRELIMINARIES}\label{sec:model}
We study a two-action contextual bandit model in which, on each round, the agent observes an input and chooses either to commit, thus executing a fixed task policy that may yield risky outcomes, or to abstain, receiving a safe default reward of zero. In this section, we introduce the formal notation and assumptions used throughout. 

For $k\in\bbn$, let $[k]=\{1,\dots,k\}$. Let $\X=\bbr^n$ be the input space, $T\in\bbn$ be the time horizon, and $\|\cdot\|$ be the Euclidean norm (though one could also consider a more general metric space).  On each time step $t\in[T]$, the agent observes an input $x_t\in\X$, chooses an action $y_t\in\{0,1\}$, and receives a (noisy) scalar reward; the precise noise assumptions are stated below. 

\paragraph{Actions and Rewards.}
We interpret $y_t=0$ as ``abstaining'', a safe default which deterministically yields $r(x_t,0)=0$ for any $x_t \in \X$. We interpret $y_t=1$ as ``committing'', which executes a preexisting policy whose reward $r(x_t,1)$ may be arbitrarily negative (catastrophic) but is assumed to have a constant upper bound (rescaled to $1$ without loss of generality). This captures the asymmetry of high-stakes settings where catastrophic losses can be unbounded in magnitude, whereas gains typically saturate. 

\paragraph{Input models.} 
We assume inputs are i.i.d.\ draws from an unknown distribution $\nu$ on $\X$, i.e., $x_1,\dots,x_T \stackrel{\text{i.i.d.}}{\sim}\nu$. We typically take $\nu$ to be fixed, but in our impossibility results we also consider the case of $T$-dependent $\nu$ (denoted $\nu_T$). 

We assume bandit feedback: the agent observes only the realized reward of its chosen action. Abstaining provides no information about the counterfactual commit reward $r(x_t,1)$, so the agent cannot ``learn by abstaining''. Formally, at round $t$ the learner observes
$$r_t = r(x_t,y_t) + \eta_t,$$
where $(\eta_t)_{t=1}^T$ are i.i.d. zero-mean $\sigma$-subgaussian noise variables, independent of $(x_t)$ and of the learner’s internal randomness (specified formally in Def.~\ref{def:subgaussian}).

% \begin{definition}[$\sigma$-subgaussian]\label{def:subgaussian}
% A random variable $Z$ is $\sigma$-subgaussian if
% $$\E[\exp(\lambda(Z-\E[Z]))] \;\le\; \exp\!\left(\tfrac{\sigma^2\lambda^2}{2}\right)
% \quad\text{for all }\lambda\in\bbr.$$
% Equivalently, $Z-\E[Z]$ has tails that are dominated by a centered Gaussian with variance proxy $\sigma^2$.
% \end{definition}

\paragraph{Regularity.}
We make two assumptions on the reward function: (i) the commit reward $r(\cdot,1)$ is $L$-Lipschitz in the Euclidean norm, i.e., there exists $L>0$ such that for all $x,x'\in\X$, $|r(x,1)-r(x',1)| \le L\|x-x'\|$. This is a standard smoothness condition in Lipschitz bandit models (see, e.g., \citealp{slivkins2019introduction}) and captures the intuition that similar inputs yield similar commit rewards. Since $r(x,0)\equiv 0$, the abstain reward is $0$-Lipschitz. (ii) The in-distribution baseline input yields strictly positive reward when committing, i.e. $r(\mathbf{0},1)>0$. This guarantees that committing is beneficial somewhere (at the origin); without it, the optimal policy would be to always abstain and cautious learning would be impossible. 

\paragraph{Objective.} The agent's goal is to minimize its (expected) regret, which is the difference between its cumulative reward and the optimal cumulative reward. Formally, define
$$\RegT = \sum_{t=1}^T \Big(\max_{y^\star\in\{0,1\}} r(x_t,y^\star)\;-\;r(x_t,y_t)\Big).$$
We take the expectations over the input process (in the stochastic model), the observation noise, and the learner’s internal randomness. The goal is to achieve sublinear expected regret, i.e., $\E[\RegT]=o(T)$, equivalently $\E[\RegT]/T\to 0$ as $T\to\infty$.

\section{THE VIRTUES AND LIMITS OF CAUTION}
\label{sec:virtues-limits}
In this section, we provide two impossibility results that demonstrate the importance and limitations of caution in high-stakes, unbounded reward bandits. 

First, \emph{caution is necessary}: if an agent commits with non-negligible probability on inputs that are far OOD, catastrophic tail losses dominate---indeed, even a single risk-agnostic exploratory commit can incur infinite expected regret. This kind of ``incautious exploration'' is exactly how standard bandit algorithms behave when they begin by pulling every arm at least once. Second, \emph{caution has limits}: when the input stream is uniformly far OOD, there is no way to explore cautiously to identify a beneficial committing region without risking catastrophe. In such settings, sublinear regret is not possible and the optimal strategy is to abstain on every time step.

\begin{theorem}[The need for caution]
\label{thm:neg-incautious}
Let $\nu$ be any distribution over $\X$ such that $\E_{x\sim \nu}[\|x\|] = \infty$ and assume $x_1,\dots,x_T \sim \nu$ i.i.d. Then there exists a reward function $r$ such that any algorithm which always commits on the first time step satisfies $\E[\RegT] = \infty$.
\end{theorem}

\begin{proof}
Define $r(x,1) = 1 - L\norm{x}$ for all $x \in \X$. Then
\begin{align*}
\E[\RegT] =&\ \E\left[\sum_{t=1}^T \left(\max_{y^*\in\{0,1\}} r(x_t,y^*) - r(x_t,y_t)\right)\right]\\
\ge&\ \E\left[\max_{y^*\in\{0,1\}} r(x_1,y^*) - r(x_1,y_1)\right]\\
\ge&\ \E[0 - (1-L\norm{x_1})]\\
=&\ L\E_{x\sim \nu}[\norm{x}]-1\\
=&\ \infty
\end{align*}
as required.
\end{proof}

The proof can easily be modified to handle the cases where the first commit is taken with constant probability (rather than probability 1) or where the algorithm abstains for a constant number of initial rounds. Essentially, this negative result applies to any algorithm that is not cautious, i.e., that explores without considering how OOD $x_t$ is. 

However, caution can only get us so far. While it prevents catastrophic first commits, some exploration is necessary to obtain sublinear regret. If all inputs are far OOD, then there is no safe way to explore, so the agent has no choice but to always abstain. Equivalently, this can be phrased by considering i.i.d.\ inputs from a $T$-dependent distribution $\nu_T$ supported on $\{x:\|x\|=T\}$. 

\begin{theorem}[The limits of caution]
\label{thm:neg-adv}
Let $\nu_T$ be any distribution supported on $\{x:\|x\|=T\}$, and suppose $x_1,\dots,x_T \stackrel{\text{i.i.d.}}{\sim} \nu_T$. Then no algorithm can guarantee $\E[\RegT]\in o(T)$.
\end{theorem}

% \begin{theorem}[The limits of caution]
% \label{thm:neg-adv}
% Suppose $\|x_t\|=T$ for all $t\in[T]$. Then no algorithm can guarantee $\E[\RegT]\in o(T)$.
% \end{theorem}

\begin{proof}

Define $r^-(x,1) := 1 - L\|x\|$ and $r^+(x,1) := 1$, with $r^\pm(x,0):=0$. Since we only care about asymptotics, we can restrict our attention to $T > 1/L$. Then for $\|x_t\| = T$, optimal behavior for $r^+$ is to always commit, while optimal behavior for $r^-$ is to always abstain. We show $\max_{r \in \{r^-, r^+\}} \E[\RegT] \in \Omega(T)$. 

To do so, we use a mild version of the probabilistic method. Let $U(r^-,r^+)$ be the uniform distribution over $\{r^-,r^+\}$. It suffices to show
$\E_{r\sim U} \E[\RegT] \in \Omega(T)$, where the second expectation is over $x_1,\dots,x_T$ and $y_1,\dots,y_T$. 
Let $\mathcal{E}$ be the event that the agent ever commits. If $\mathcal{E}$ holds, there exists $i\in[T]$ with $y_i=1$. Since $y_i$ is independent of $r$,
% \begin{align*}
% &\ \E_r\ \E[\RegT \mid \mathcal{E}]
% \\ =&\ \E_r\ \E\left[\sum_{t=1}^T \left(\max_{y^*\in\{0,1\}} r(x_t,y^*) - r(x_t,y_t)\right) \mid \mathcal{E}\right]\\
% \ge&\ \E_r\  \E\left[\max_{y^*\in\{0,1\}} r(x_i,y^*) - r(x_i,1)\right]\\
% \ge&\ \Pr[r = r^-] \E\left[\max_{y^*\in\{0,1\}} r(x_i,y^*) - r(x_i,y_i)\mid r=r^-\right]\\
% \ge&\ \frac{1}{2}\E\left[\max_{y^*\in\{0,1\}} r^-(x_i,y^*) - r^-(x_i,y_i)\right]\\
% \ge&\ \frac{1}{2}\E\left[0 - (1-L\norm{x_i})\right]\\
% =&\ \frac{LT-1}{2}
% \end{align*}
\begin{align*}
&\ \E_r\ \E[\RegT \mid \mathcal{E}]
\\ =&\ \E_r\ \E\left[\sum_{t=1}^T \left(\max_{y^*\in\{0,1\}} r(x_t,y^*) - r(x_t,y_t)\right) \mid \mathcal{E}\right]\\
% \ge&\ \E_r\  \E\left[\max_{y^*\in\{0,1\}} r(x_i,y^*) - r(x_i,1)\right]\\
\ge&\ \Pr[r = r^-] \E\left[\max_{y^*\in\{0,1\}} r(x_i,y^*) - r(x_i,y_i)\mid r=r^-\right]\\
% \ge&\ \frac{1}{2}\E\left[\max_{y^*\in\{0,1\}} r^-(x_i,y^*) - r^-(x_i,y_i)\right]\\
% \ge&\ \frac{1}{2}\E\left[0 - (1-L\norm{x_i})\right]\\
=&\ \frac{LT-1}{2}
\end{align*}
On the other hand, if $\mathcal{E}$ does not occur, then
% \begin{align*}
% &\ \E_r\ \E[\RegT \mid \neg\mathcal{E}]\\
% =&\ \E_r\ \E\left[\sum_{t=1}^T \left(\max_{y^*\in\{0,1\}} r(x_t,y^*) - r(x_t,y_t)\right) \mid \neg\mathcal{E}\right]\\
% \ge&\ \Pr[r = r^+] \E\left[\sum_{t=1}^T \left(\max_{y^*\in\{0,1\}} r^+(x_t,y^*) - r^+(x_t,y_t)\right) \mid \neg\mathcal{E}\right]\\
% =&\ \frac{1}{2}\E\left[\sum_{t=1}^T \left(\max_{y^*\in\{0,1\}} r^+(x_t,y^*) - r^+(x_t,0)\right)\right]\\
% % \ge&\ \frac{1}{2}\E\left[\sum_{t=1}^T (1 - 0)\right]\\
% % =&\ \frac{T}{2}
% \ge&\ \frac{T}{2}
% \end{align*}
\begin{align*}
&\ \E_r\ \E[\RegT \mid \neg\mathcal{E}]\\
\ge\;& \Pr[r = r^+] \E\Biggl[
\sum_{t=1}^T \Bigl(
\max_{y^*\in\{0,1\}} r^+(x_t,y^*) - r^+(x_t,y_t)
\Bigr) \\
&\qquad\qquad\Bigm| \neg\mathcal{E}
\Biggr] \\
\ge&\  \frac{T}{2}.
\end{align*}

Then by the law of total expectation,
\begin{align*}
&\ \E_r\  \E[\RegT]\\
=&\ \Pr[\mathcal{E}]\ \E_r\ \E[\RegT \mid \mathcal{E}] + \Pr[\neg \mathcal{E}]\  \E_r\ \E[\RegT \mid \neg\mathcal{E}]\\
\ge&\ \min\big(\Pr[\mathcal{E}] , \Pr[\neg \mathcal{E}]\big) \min\left(\frac{LT-1}{2},\, \frac{T}{2}\right)\\
\in&\ \Omega(T)
\end{align*}
as required.
\end{proof}

\section{ALGORITHM AND MAIN RESULT}\label{sec:algorithm}

Following the negative results in \autoref{sec:virtues-limits}, we propose an algorithm (Algorithm~\ref{alg:main}) that operationalizes cautious learning: only learn in regions that are not too far OOD and where the available evidence does not already certify that committing is harmful. 

Informally, we define a trusted region around the origin whose radius grows with the time horizon, reflecting the maximum regret we are willing to tolerate---intuitively, this corresponds to allowing mistakes that are bad but not catastrophic. We then discretize the region into bins to exploit Lipschitz continuity. Within each bin, the commit reward cannot vary by more than a Lipschitz discretization error, so it suffices to estimate a single per-bin mean. The agent always abstains outside the trusted region, and inside it abstains in any bin whose pessimistic upper bound on reward is negative; otherwise it commits to gather information. 

More precisely,  the algorithm defines a ball of radius $m(T)$ around the origin, treating inputs outside this ball as too OOD to test. The ball is partitioned into $n$-dimensional hypercubes (bins) of side length $w(T)$. By Lipschitz continuity, the variation of $r(\cdot,1)$ within any bin $B$ is at most $L\sqrt{n} w(T)$. For each bin $B$, the algorithm maintains its empirical mean $\hat\mu_B$ and a confidence radius $\gamma(k)$ after $k$ commits in $B$. If $\hat\mu_B+\gamma(k)+L\sqrt{n} w(T)<0$, then $B$ is certified unsafe and the algorithm abstains there permanently. Figure~\ref{fig:schematic} shows a schematic of the algorithm. 

\begin{algorithm}[t]
\caption{Risk-Sensitive Abstention Algorithm}
\label{alg:main}
\begin{algorithmic}
\State Inputs: $m: \bbn \to \bbrspos,\: w:\bbn\to\bbrspos$
\State $\mathcal{H} \gets$ partition of $\X$ into $n$-cubes of side length $w(T)$
\State $\B \gets$ $\{B \in \mathcal{H}: \exists x \in B \text{ with }\norm{x} \le m(T)\}$
% \State $E \gets$ union of hypercubes fully contained with $\{x\in \X: \norm{x} \le m(T)\}$.
\State $\sigma_w \gets \sqrt{n L^2w(T)^2 + \sigma^2}$
\State $\gamma(k) := \infty$ if $k=0$ else $\sqrt{\frac{c^{-1}\sigma_w^2\ln(2T^4)}{k}}$ where $c$ is the absolute constant from Lemma~\ref{lem:concentrate}
\State $(k_B, \hat{\mu}_B) = (0,0)$ for all $B \in \B$
\For{$t=1, \dots, T$}
 \If{$\exists B \in \B$ s.t. $x_t \in B$}
 \State $\triangleright$ \emph{$x_t$ is not too OOD: it's safe to learn}
 \If{$\hat{\mu}_B + \gamma(k_B) + L\sqrt{n}w(T)< 0$}
  \State $\triangleright$ \emph{We already know $x_t$ is bad: don't learn}
    \State Abstain ($y_t = 0$)
 \Else
 \State $\triangleright$ \emph{$x_t$ might be good: learn}
  \State Commit ($y_t=1$)
  \State $k_B \gets k_B+1$
  \State $\hat{\mu}_B \gets \hat{\mu}_B + \frac{r_t - \hat{\mu}_B}{k_B}$
 \EndIf
 \Else
  \State $\triangleright$ \emph{$x_t$ is far OOD: it's too risky to learn}
  \State Abstain ($y_t = 0$)
 \EndIf
\EndFor
\end{algorithmic}
\end{algorithm}

\begin{figure}[t]
\centering
\begin{tikzpicture}[scale=1.0]
  %==============================
  % Parameters
  %==============================
  \def\mT{2.8}   % trusted radius m(T)
  \def\wT{0.8}   % bin side length
  \def\R{3.6}    % half-width of plotting window

  %==============================
  % Styles
  %==============================
  \tikzstyle{gridline}=[gray!60, line width=0.3pt]
  \tikzstyle{ball}=[draw=blue!60, line width=1pt]
  \tikzstyle{trustedfill}=[fill=green!12]
  \tikzstyle{outerfill}=[fill=red!6]
  \tikzstyle{unsafe}=[pattern=north east lines, pattern color=red!70, draw=red!70, line width=0.5pt]
  \tikzstyle{label}=[font=\footnotesize]
  \tikzstyle{legendbox}=[draw=black!60, line width=0.4pt, fill=white, rounded corners=2pt, inner sep=2pt]

  %==============================
  % Background: blocked red outside the trusted ball
  %==============================
  \fill[outerfill] (-\R,-\R) rectangle (\R,\R);

  %==============================
  % Fill trusted bins: any square that intersects the circle is fully green
  %==============================
  \foreach \i in {-3.2,-2.4,-1.6,-0.8,0,0.8,1.6,2.4,3.2}{
    \foreach \j in {-3.2,-2.4,-1.6,-0.8,0,0.8,1.6,2.4,3.2}{
      % Closest point on the square [\i,\i+\wT]x[\j,\j+\wT] to (0,0):
      \pgfmathsetmacro{\xc}{max(\i, min(0, \i+\wT))}
      \pgfmathsetmacro{\yc}{max(\j, min(0, \j+\wT))}
      % Intersects if distance to closest point <= mT:
      \pgfmathparse{ (sqrt(\xc*\xc + \yc*\yc) <= \mT) ? 1 : 0 }
      \ifnum\pgfmathresult>0
        \fill[trustedfill] (\i,\j) rectangle ++(\wT,\wT);
      \fi
    }
  }

  %==============================
  % Grid lines (bins) — on top of fills
  %==============================
  \foreach \x in {-3.2,-2.4,-1.6,-0.8,0,0.8,1.6,2.4,3.2}{
    \draw[gridline] (\x,-\R) -- (\x,\R);
  }
  \foreach \y in {-3.2,-2.4,-1.6,-0.8,0,0.8,1.6,2.4,3.2}{
    \draw[gridline] (-\R,\y) -- (\R,\y);
  }

  %==============================
  % Trusted ball outline (for reference)
  %==============================
  \draw[ball] (0,0) circle (\mT);

  %==============================
  % Certified-unsafe bins (hatched red) — full-square overlay
  %   (edit the list below to choose which bins are certified)
  %==============================
  \foreach \i/\j in { -0.8/0.8, 0.8/-0.8, 1.6/0 }{
    \path[unsafe] (\i,\j) rectangle ++(\wT,\wT);
  }

  %==============================
  % Origin
  %==============================
  \fill[black] (0,0) circle (3pt);
  \node[label] at (0.25,0.18) {$\mathbf{0}$};

  %==============================
  % Axes (draw last so they're on top)
  %==============================
  \draw[->, gray!95] (-\R,0) -- (\R,0) node[below right, label] {$x_1$};
  \draw[->, gray!95] (0,-\R) -- (0,\R) node[above left, label] {$x_2$};

  %==============================
  % Bin side label w(T)
  %==============================
  \draw[<->, label] (-3.2,-2.6) -- (-2.4,-2.6);
  \node[label] at (-2.8,-2.85) {$w(T)$};

  %==============================
  % Legend (as requested)
  %==============================
  \begin{scope}[shift={(2.55, -2.1)}]
    \node[legendbox] (L) {
      \begin{tikzpicture}
        % grid (bins)
        \draw[gridline] (0,0) rectangle (0.35,0.25);
        \node[anchor=west] at (0.45,0.12) {\footnotesize grid (bins)};
        % bin in B (trusted)
        \fill[trustedfill] (0,-0.35) rectangle (0.35,-0.1);
        \draw[gridline] (0,-0.35) rectangle (0.35,-0.1);
        \node[anchor=west] at (0.45,-0.22) {\footnotesize bin $\in \mathcal{B}$ (trusted)};
        % certified negative
        \path[unsafe] (0,-0.7) rectangle (0.35,-0.45);
        \node[anchor=west] at (0.45,-0.57) {\footnotesize certified negative};
        % ball label
        \draw[ball] (0.18,-1.0) circle (0.12);
        \node[anchor=west] at (0.45,-1.0) {\footnotesize $\|x\|\le m(T)$};
      \end{tikzpicture}
    };
  \end{scope}

\end{tikzpicture}
\caption{Trusted region of radius $m(T)$ around the origin, partitioned into bins of side $w(T)$. Any square intersecting the ball is shown fully green (bin $\in\mathcal B$). Certified negative bins are shown hatched red. The agent abstains outside the ball.}
\label{fig:schematic}
\end{figure}

We saw in \autoref{sec:virtues-limits} that the problem is impossible when inputs are too far OOD. A natural way to quantify this is via the amount of probability mass that lies outside a given radius, captured by the \emph{radial survival function}:
\begin{definition}[Radial survival function]
For any radius $R\ge 0$, the radial survival function of $\nu$ is $\bar\nu(R) := \Pr_{x\sim\nu}\big[\|x\|\ge R\big]$.
\end{definition}

 We are now ready to state our main result.

\begin{restatable}{theorem}{thmMain}
\label{thm:main}
In the stochastic setting with $x_t\sim\nu$ i.i.d., Algorithm~\ref{alg:main} with $w(T)=T^{-1/(n+2)}$ and $m(T)=\ln T$ satisfies
$$
\E[\RegT] \in O\left((L +\sigma^2)  T^{\frac{n+1}{n+2}} (\ln T)^{ n+1} + T \bar\nu(\ln T)\right).$$
\end{restatable}

The first term is typical for Lipschitz contextual bandits and reflects the curse of dimensionality (see, e.g., \citealp[Thms.~4.11--4.12]{slivkins2019introduction};~\citealp[Thm.~10]{plaut2025safe}).

The $T\bar{\nu}(\ln T)$ term is unusual mainly because unbounded domains are unusual: for bounded domains, $\bar{\nu}(\ln T)=0$ for all large $T$. Our analysis deals with far OOD inputs directly and the bound necessarily degrades as such inputs become more frequent. This dependence is unavoidable: the construction in \Cref{thm:neg-adv} sets $x_t=T$ for all $t\in[T]$, hence $\bar{\nu}(\ln T)=1$ and the bound in \Cref{thm:main} becomes linear, matching the impossibility result. By contrast, for any fixed distribution $\nu$, we have $\bar{\nu}(\ln T)\to 0$ as $T\to\infty$, so $T\bar{\nu}(\ln T)=o(T)$ and the overall regret stays sublinear. 

For example, if $\nu$ is subgaussian with $\bar\nu(r)\le e^{-c r^2}$, then $T\bar\nu(\ln T)\le T e^{-c(\ln T)^2}=o(1)$. If $\nu$ is subexponential with $\bar\nu(r)\le e^{-c r}$, then $T\bar\nu(\ln T)\le T\cdot T^{-c}=T^{1-c}=o(T)$. If $\nu$ has polynomial tails with $\bar\nu(r)\asymp r^{-\alpha}$ for $\alpha>0$, then $T\bar\nu(\ln T)\asymp T/(\ln T)^\alpha=o(T)$. If one has prior knowledge of $\nu$, the choice of $m(T)$ can be tailored more precisely than our generic setting $m(T)=\ln T$. In particular, for polynomial tails, setting $m(T)=T^c$ for small $c>0$ improves the bound to $O(T^{1-c\alpha})$. 

Thus, the regret decomposes into a geometric/statistical term from discretization and concentration inside the trusted region, and a tail term from far OOD inputs; both are sublinear for any fixed $\nu$.

\section{PROOF SKETCH}\label{sec:roadmap}

We now outline the logical structure of the proof of~\Cref{thm:main}; we also provide the intuition behind each step. Full technical details and complete proofs of all lemmas are deferred to~\autoref{app:proofs}.

Let $m(T),w(T)$ be as in Algorithm~\ref{alg:main}, and let $\B$ be the set of bins intersecting the ball of radius $m(T)$. For any $B\in\B$, let
$\mu_B=\E_{x\sim\nu}[\,r(x,1)\mid x\in B\,]$ be its true mean commit reward, let $k_B(t)$ be the number of commits taken in $B$ by the end of round $t$ (so $k_B(0)=0$), and let $\hat\mu_B(k)$ be the empirical mean in $B$ after $k$ commits (i.e., the running mean from Algorithm~\ref{alg:main} indexed by its commit count).

To control estimation error we define the confidence radius
$$\gamma(k)=\sqrt{\frac{c^{-1}\sigma_w^2\ln(2T^4)}{k}},
\qquad 
\sigma_w^2 = nL^2 w(T)^2 + \sigma^2,
$$
where $c>0$ is the absolute constant from Lemma~\ref{lem:concentrate}. Here $\sigma_w^2$ aggregates the combines observation noise $\sigma^2$ with the Lipschitz-induced within-bin variation of $(L\sqrt{n}w(T))^2$.

\begin{definition}[$\sigma$-subgaussian]\label{def:subgaussian}
A random variable $Z$ is $\sigma$-subgaussian if
$$\E[\exp(\lambda(Z-\E[Z]))] \;\le\; \exp\!\left(\tfrac{\sigma^2\lambda^2}{2}\right)
\quad\text{for all }\lambda\in\bbr.$$
Equivalently, $Z-\E[Z]$ has tails that are dominated by a centered Gaussian with variance proxy $\sigma^2$.
\end{definition}
Define the ``good'' event under which all per-bin estimates are accurate over the realized commit counts: 
\begin{align*}
\G = \Big\{& \forall B\in\B,\ \forall t\in[T] \text{ s.t. } k_B(t) >0:\\
& |\hat\mu_B(k_B(t))-\mu_B| \le \gamma(k_B(t)) \Big\}.
\end{align*}
On $\G$, each empirical mean is a reliable proxy for its bin’s true mean at every commit count that occurs along the algorithm’s trajectory. The analysis conditions on $\G$ (which holds with high probability by a union bound over realized bin–count pairs), and then decomposes regret into: (i) commits inside the trusted region (handled by certification of negative bins plus a margin term), and (ii) abstentions outside the ball of radius $m(T)$ (quantified by the radial survival function).

\begin{restatable}[Per-bin concentration]{lemma}{lemPerBinConcentration}\label{lem:per-bin-concentration}
For any bin $B \in \B$ and time step $t \in [T]$ such that $k_B(t) > 0$, $\Pr\left[|\hmub(k_B(t)) - \mu_B| > \gamma(k_B(t))\right] \le T^{-3}$.
\end{restatable}

\begin{proofidea}
In Algorithm~\ref{alg:main}, whether we commit depends only on the history (and the bin of the current context). Thus we can use a ``reward tape'' argument: imagine pre-sampling an infinite list of rewards for each bin and revealing the next one whenever the algorithm commits in that bin. Then the sequence of committed rewards within any fixed bin $B$ is distributionally equivalent to an i.i.d. sequence with $x_i \sim \nu(\cdot \mid x \in B)$ and independent noise.

%so by the ``reward tape'' \ben{I think we can't just use ``reward tape'' without definition, most people won't know what the means} argument the sequence of committed rewards within any fixed bin $B$ is distributionally equivalent to an i.i.d. sequence with $x_i \sim \nu(\cdot \mid x \in B)$ and independent noise. For a fixed commit count $k$, decompose
\begin{align*}
\hat\mu_B(k)-\mu_B =&\ \frac{1}{k} \sum_{i=1}^{k}\big(r(x_{i},1)-\mu_B\big) + \frac{1}{k}\sum_{i=1}^{k} \eta_{i},
\end{align*}
where the first term captures within-bin reward variation and the second term is $\sigma$-subgaussian observation noise. Since all $x_i \in B$, \Cref{lem:bin-bounds} gives $|r(x_i,1)-\mu_B| \le L\sqrt{n}w(T)$, so these terms are $(L\sqrt{n}w(T))$-subgaussian; together with the noise, the sum is $\sigma_w$-subgaussian. A standard tail bound yields failure probability at most $T^{-4}$ for fixed $k$. Finally, union bound over $k \in \{1,\dots,T\}$ to handle the random count $k_B(t)$, giving $T\cdot T^{-4}=T^{-3}$.
\end{proofidea}

% \ben{This needs to be updated}
% Fix $B$ and $t$ and condition on the $k_B(t)$ commit times $t_1<\cdots<t_{k_B(t)}$ with $x_{t_j}\in B$. Decompose
% \begin{align*}
% \hat\mu_B(k_B(t))-\mu_B =&\ \\
% \frac{1}{k_B(t)} \sum_{j=1}^{k_B(t)}\big(r(x_{t_j},1)-\mu_B\big)&\ +\frac{1}{k_B(t)}\sum_{j=1}^{k_B(t)} \eta_{t_j}.
% \end{align*}

% By Lipschitz continuity and the fact that all $x_{t_j}$ lie in a single cube of side $w(T)$, we have
% $\big|r(x_{t_j},1)-\mu_B\big|\le L\sqrt{n} w(T)$ (see \Cref{lem:bin-bounds}),
% so the first term is bounded and hence subgaussian with variance proxy $O((L\sqrt{n} w(T))^2)$.
% The second term is the observation noise, which is $\sigma$-subgaussian by assumption.
% Standard subgaussian tail bounds then give
% $\Pr\left(|\hat\mu_B(k_B(t))-\mu_B|>\gamma(k_B(t))\right)\le T^{-4}$.

\Cref{lem:per-bin-concentration} gave concentration for each fixed bin and commit count. To extend this guarantee uniformly, we apply a union bound over the at most $T$ bin–time pairs actually realized by the algorithm.

\begin{restatable}[Uniform concentration bound]{lemma}{lemUniformConcentration}
\label{lem:uniform_concentration}
With probability at least $1-1/T$, the good event $\G$ holds. Equivalently,
$\Pr(\neg\G)\le 1/T $.
\end{restatable}

\begin{proofidea}
    There are at most $T^2$ relevant pairs $(B,t)$ along the trajectory of the algorithm. Each has failure probability $T^{-3}$ by \Cref{lem:per-bin-concentration}. A union bound yields failure probability at most $T^{-1}$.
\end{proofidea} 

Recall that the algorithm abstains permanently in any bin once $\hat\mu_B(k_B(t))+\gamma(k_B(t))+L\sqrt{n} w(T)<0$. On the good event $\G$, bins with sufficiently negative mean are thus certified unsafe after finitely many commits. (For bins near the decision boundary, certification may not occur, but their per-round regret is $O(L\sqrt{n} w(T))$, so their total contribution is small and accounted for by the margin term later.)
We now compute how many commits are needed to certify a negative bin.

\begin{restatable}[Samples for negative certification]{lemma}{lemSamplesForCert}\label{lem:samples_for_cert}
Consider any $t \in [T]$ and $B \in \B$. On $\G$, if $\mu_B < - L\sqrt{n}w(T)$ and $k_B(t) > \frac{4c^{-1}\sigma_w^2 \ln(2T^4)} {(\mu_B + L\sqrt{n}w(T))^2}$, then bin $B$ is certified negative at time $t$.
\end{restatable}
\begin{proofidea}
On the good event $\G$, certification in bin $B$ occurs when $\hat\mu_B+\gamma(k_B(t))+L\sqrt{n} w(T)<0$.
Using the worst-case deviation $\hat\mu_B=\mu_B+\gamma(k_B(t))$, this reduces to
$\mu_B+2\gamma(k_B(t))+L\sqrt{n} w(T)<0$.
Plugging in the definition of $\gamma(k_B(t))$ and solving for $k$ yields the stated number of commits needed for certification.
\end{proofidea}

Next, we bound the geometry of the trusted region. By construction, $\B$ consists of all bins intersecting the ball of radius $m(T)$. Consequently, their union $\bigcup_{B\in\B}B$ is contained within a slightly larger ball. This enlarged region will be useful both for bounding how negative rewards can be (via Lipschitz continuity) and for controlling the number of bins (via volume packing). 

Let $v_1$ be the volume of the unit ball $\{x \in \X : \|x\| \le 1\}$. Let $R(T) = m(T) + \sqrt{n} w(T)$, which is the maximum distance from the origin to any point in the trusted region, as shown by the following lemma:

\begin{restatable}[Trusted region is a slightly larger ball]{lemma}{lemBall}\label{lem:ball}
Every $x\in\bigcup_{B\in\B}B$ satisfies $\|x\|\le R(T)=m(T)+\sqrt{n} w(T)$.
\end{restatable}

We now bound the regret from a truly unsafe bin before it is certified. Let $\Delta_t:=\max_{y\in\{0,1\}}r(x_t,y)-r(x_t,y_t)$ be the instantaneous regret at time $t$.

\begin{restatable}[Per-bin commit regret]{lemma}{lemRegretPerBin}\label{lem:regret-per-bin}
On $\G$, for any $B\in\B$ with $k_B(T)\ge 1$ and $\mu_B<-(2L\sqrt{n}+1)w(T)$,
\[
\sum_{t:\ x_t\in B,\ y_t=1}\Delta_t
\le 2L R(T)+\frac{32c^{-1}\sigma_w^2\ln(2T^4)}{w(T)}.
\]
\end{restatable}

\begin{proofidea}
\Cref{lem:samples_for_cert} shows that a negative bin is certified after $O(\sigma_w^2/(\mu_B+L\sqrt{n} w(T))^2)$ commits. Each such commit incurs at most $O(|\mu_B|)$ regret, but \Cref{lem:ball} ensures that $\mu_B\ge -L R(T)$, so the loss per commit is bounded. Multiplying the number of pre-certification commits by the maximum per-step regret yields the stated bound.
\end{proofidea}

Now that we have controlled the regret contribution of each individual bin, we sum across all bins that are ever visited and include the effect of near-margin bins (those with $\mu_B$ close to zero). Such bins may never be certified, but their regret per commit is small, so their total contribution is still controlled. 

\begin{restatable}[Total commit regret inside the trusted region]{lemma}{lemTotalCommit}\label{lem:total-commit-regret}
On $\G$, 
\begin{align*}
\sum_{t:\ y_t=1}\Delta_t &\le \frac{v_1 R(T)^n}{w(T)^n}\left( 2LR(T)+\frac{32c^{-1}\sigma_w^2\ln(2T^4)}{w(T)} \right) \\
&\quad+(3L\sqrt{n}+1) w(T) T.
\end{align*}
\end{restatable}

\begin{proofidea}
We partition commits into bins with decisively negative mean and those near the decision boundary. For $\mu_B$ well below zero, \Cref{lem:regret-per-bin} bounds the regret before certification. Summing over all such bins gives at most $|\B|$ times the per-bin cost, and by the packing bound $|\B| w(T)^n\le v_1 R(T)^n$, this gives the first term. For bins near the margin, the algorithm may continue committing longer, but Lipschitzness bounds the per-round regret by $(3L\sqrt{n}+1)w(T)$, giving the second term after $T$ rounds. 
\end{proofidea}

\Cref{lem:total-commit-regret} completes the analysis of commit regret inside the trusted region. Combining this with the abstention regret outside the ball of radius $m(T)$, we yield the final rate in~\Cref{thm:main} as follows.

\begin{proofidea}[ of~\Cref{thm:main}]
Regret decomposes into (i) abstention outside the trusted ball; (ii) commits inside. 

For (i), each input with $\|x_t\|>m(T)$ contributes at most $1$, giving $T\bar\nu(m(T))$, which is sublinear for any fixed $\nu$ since $\bar\nu(\ln T)\to 0$. For (ii), Lipschitz continuity bounds the within-bin variation, and empirical means stay within confidence radii under the highly likely good event $\G$. Negative bins are certified after $O(\sigma_w^2/\text{margin}^2)$ commits, so each contributes at most $O(LR(T)+\sigma_w^2\log T/w(T))$ regret. Summing across $O((m(T)/w(T))^n)$ bins gives
$$\tilde O\left(R(T)^n \big(LR(T)w(T)^{-n} + \sigma_w^2 w(T)^{-(n+1)}\big) \right),$$ and bins near the decision boundary contribute an additional $O(w(T)T)$. 
\end{proofidea}

If we ignore log factors, the leading terms trade off
$$\underbrace{R(T)^n w(T)^{-(n+1)}}_{\text{variance-driven}}
\quad\text{vs.}\quad
\underbrace{w(T) T}_{\text{margin-driven}}.$$
Balancing these yields the optimal choice $w(T)\asymp T^{-1/(n+2)}$. Independently, the radius $m(T)$ trades off the abstention term $T\bar{\nu}(m(T))$ against the growth of the volume factor $R(T)^n$. Choosing $m(T)=\ln T$ makes $T\bar{\nu}(m(T))$ sublinear for any fixed $\nu$ (since $\bar{\nu}(\ln T) \to 0$) while increasing $R(T)$ only logarithmically.

\section{CONCLUSION}
In this work, we introduced a formal model for safe learning under distribution shift in contextual bandits with catastrophic tails, provided impossibility results that clarify when sublinear regret is unattainable, and gave a cautious risk-sensitive algorithm with sublinear regret under suitable conditions. Our work has several limitations, which also provide directions for future work. These include relaxing Lipschitz continuity, extending the analysis to non-i.i.d. inputs or worst-case sequences, and incorporating additional structure that could support stronger guarantees or improved rates. \looseness-1
%Our work has several limitations, which also provide directions for future work. These include relaxing Lipschitz continuity, incorporating adaptive or learned metrics, \ben{What does ``incorporating adaptive or learned metrics'' mean? Like making the algorith more adaptive?} \sarah{and extending the analysis to non-i.i.d. inputs or worst-case sequences, and extending the model beyond the binary abstain/commit setting to multiple task actions.} \looseness-1 \ben{tbh I feel like extending to multiple task actions is not that interesting. But we can still mention it if you want}

\paragraph{Our regret bound can be close to linear.}
In~\Cref{thm:main}, the abstention term $T\bar\nu(\ln T)$ can dominate for heavy-tailed inputs (e.g., power laws). This is the price of caution: avoiding catastrophic far OOD commits requires systematic abstention in the tails, and the resulting regret can be unavoidable (see the impossibility in~\Cref{thm:neg-adv}). Moreover, while the bound is sublinear for every fixed $n$, the exponent $(n+1)/(n+2)\to 1$ as $n\to\infty$, which is a standard curse of dimensionality in Lipschitz contextual bandits \citep[Thms.~4.11–4.12]{slivkins2019introduction}; see also \citet[Thm.~10]{plaut2025safe}. While we do not expect to remove these dependencies entirely, future work could improve rates. The simplicity of \Cref{alg:main} is appealing but ignores useful information: commits inform not only their own bin but also nearby bins via Lipschitz continuity, and certifying a bin as positive could justify expanding the trusted region around it. Additional structural assumptions, such as margin/low-noise conditions, intrinsic low dimensionality, or smoothness beyond Lipschitz, could also help.

\paragraph{Assumptions may not always hold.} Our guarantees here rely on i.i.d. inputs and Lipschitz continuity of the commit reward. In practice, inputs may drift or exhibit temporal dependence, and rewards may be only piecewise smooth or even non-smooth. Extending the analysis to weaker smoothness conditions or drifting processes is an important direction. Moreover, \Cref{alg:main} assumes knowledge of $L$, $\sigma^2$, and $T$. While knowledge of $T$ can be handled by the standard doubling trick (see \citet[\S1.5]{slivkins2019introduction}), $L$ and $\sigma^2$ may be unknown. Thus, developing parameter-free (or adaptively tuned) algorithms that remain cautious would increase robustness
%\sarah{We also focused on the binary abstain/commit setting for simplicity of exposition; in principle, the same cautious-learning idea could be extended to multiple task actions by maintaining action-specific estimates within each bin and applying a union bound across actions.} \ben{Yeah this just feels fairly straightforward to me for precisely the union bound reason you mention}

\paragraph{No unconditionally irreparable errors.} Obtaining regret $-T$ on a single time step is irreparable in the sense that it automatically implies linear regret on that run. However, errors in our model are only irreparable for a fixed $T$: for any error, there exists a large enough $T$ that the error is no longer catastrophic. It may be worth considering alternative models of catastrophe such as inescapable trap states in MDPs which do allow for errors that are unconditionally catastrophic. 

\paragraph{Broader impact.} This work is motivated by safety concerns in the deployment of learning systems in high-stakes domains. We provide theoretical justification for abstention as a mechanism for averting catastrophic errors under distribution shift, and abstention is also a practical choice for deployed systems. Agents that can defer action when uncertain may be safer and more trustworthy, but abstention mechanisms must be designed carefully to avoid consequences such as excessive conservatism or over-reliance on human supervision.

\section*{Acknowledgments}
This work was supported by a gift from Open Philanthropy to the Center for Human-Compatible AI (CHAI) at UC Berkeley. We thank Vamshi Bonagiri and Pavel Czempin for their insightful discussions and valuable feedback.

S. Liaw and B. Plaut contributed equally as co-first authors to this work. S. Liaw contributed to the early development of the project, including exploring several initial ideas and leading the development of the proof structure and intermediate arguments. The project was conceived and supervised by B. Plaut, who developed the final direction of the work and completed the final proof. S. Liaw wrote the first draft of the manuscript, with B. Plaut providing detailed feedback.

% \input{main_result}

% \onecolumn
% \input{main_proof}

\clearpage

\bibliographystyle{apalike}
\bibliography{ref}

%%%%%%%%%%%%%%%%%%%%%%%%%%%%%%%%%%%%%%%%%%%%%%%%%%%%%%%%%%%%

\section*{Checklist}

\begin{enumerate}

  \item For all models and algorithms presented, check if you include:
  \begin{enumerate}
    \item A clear description of the mathematical setting, assumptions, algorithm, and/or model. \textbf{Yes}, \autoref{sec:model} and \autoref{sec:algorithm}.
    \item An analysis of the properties and complexity (time, space, sample size) of any algorithm. \textbf{Yes}, \autoref{sec:algorithm} and \autoref{app:proofs}.
    \item (Optional) Anonymized source code, with specification of all dependencies, including external libraries. \textbf{Not Applicable}, no code was used for this work.
  \end{enumerate}

  \item For any theoretical claim, check if you include:
  \begin{enumerate}
    \item Statements of the full set of assumptions of all theoretical results. \textbf{Yes}, \autoref{sec:model} and \autoref{sec:algorithm}.
    \item Complete proofs of all theoretical results. \textbf{Yes}, \autoref{app:proofs}.
    \item Clear explanations of any assumptions. \textbf{Yes}, \autoref{sec:model}.
  \end{enumerate}

  \item For all figures and tables that present empirical results, check if you include:
  \begin{enumerate}
    \item The code, data, and instructions needed to reproduce the main experimental results (either in the supplemental material or as a URL). \textbf{Yes}, \autoref{app:sim}.
    \item All the training details (e.g., data splits, hyperparameters, how they were chosen). \textbf{Yes}, \autoref{app:sim}.
    \item A clear definition of the specific measure or statistics and error bars (e.g., with respect to the random seed after running experiments multiple times). \textbf{Yes}, \autoref{app:sim}.
    \item A description of the computing infrastructure used. (e.g., type of GPUs, internal cluster, or cloud provider). \textbf{Yes}, \autoref{app:sim}.
  \end{enumerate}

  \item If you are using existing assets (e.g., code, data, models) or curating/releasing new assets, check if you include:
  \begin{enumerate}
    \item Citations of the creator If your work uses existing assets. \textbf{Not Applicable}, no use of existing assets or release of new assets.
    \item The license information of the assets, if applicable. \textbf{Not Applicable}, no use of existing assets or release of new assets.
    \item New assets either in the supplemental material or as a URL, if applicable. \textbf{Not Applicable}, no use of existing assets or release of new assets.
    \item Information about consent from data providers/curators. \textbf{Not Applicable}, no use of existing assets or release of new assets.
    \item Discussion of sensible content if applicable, e.g., personally identifiable information or offensive content. \textbf{Not Applicable}, no use of existing assets or release of new assets.
  \end{enumerate}

  \item If you used crowdsourcing or conducted research with human subjects, check if you include:
  \begin{enumerate}
    \item The full text of instructions given to participants and screenshots. \textbf{Not Applicable}, no research with human subjects.
    \item Descriptions of potential participant risks, with links to Institutional Review Board (IRB) approvals if applicable. \textbf{Not Applicable}, no research with human subjects.
    \item The estimated hourly wage paid to participants and the total amount spent on participant compensation. \textbf{Not Applicable}, no research with human subjects.
  \end{enumerate}

\end{enumerate}

\clearpage
\onecolumn
\appendix
\section{PROOF OF MAIN RESULT}\label{app:proofs}

\subsection{Proof notation}

The proof will use the following notation:
\begin{enumerate}[
    topsep=0.5ex,
    partopsep=0pt,
    parsep=0pt,
    itemsep=0.5ex]
    \item The true mean of bin $B$ is $\mu_B=\mathbb{E}_{x\sim\nu}[r(x,1)\mid x\in B]$.
    \item Let $k_B(t)$ denote the value of the variable $k_B$ in Algorithm~\ref{alg:main} at the end of time step $t$.
    % \item For each bin $B \in \B$, let $\C_B(t) = \{i\in [t]:x_i \in B, y_i = 1\}$ be the set of time steps up until $t$ on which we commit on an input in $B$. Note that $|\C_B(t)| = k_B(t)$.
    \item  Let $\hmub(k)$ denote the value of the variable $\hmub$ in Algorithm~\ref{alg:main} after $k$ rewards from bin $B$ have been observed.
    \item Let $\sigma_w = \sqrt{n L^2w(T)^2 + \sigma^2}$ for brevity.
    \item Define the confidence radius $\gamma(k) = \sqrt{\frac{c^{-1}\sigma_w^2\ln(2T^4)}{k}}$ where $c$ is the absolute constant from Lemma~\ref{lem:concentrate}.
    \item Define the good event $\G = \{\forall t\in[T], \forall B \in \B \text{ where } k_B(t) > 0: |\hmub(k_B(t)) - \mu_B| \le \gamma(k_B(t))\}$.
    \item A bin $B$ is certified negative at time $t$ if $\hmub(k_B(t)) + \gamma(k_B(t)) + L\sqrt{n}w(T) < 0$. %A bin $B$ is certified positive at time $t$ if $\hmub(k_B(t)) - \gamma(k_B(t)) - L\sqrt{n}w(T) > 0$.
    \item Let $\bar{\nu}$ be the radial survival function of $\nu$. That is, for any $y \in \bbrpos$, $\bar{\nu}(y) = \Pr_{x\sim \nu}[\norm{x} \ge y]$.
    \item Let $\Delta_t = \max_{y^* \in \{0,1\}} r(x_t,y^*) - r(x_t,y_t)$ be the single-step regret at time $t$.
    \item Let $v_1$ be the volume of the unit ball $\{x\in\X: \norm{x} \le 1\}$.
    \item Let $R(T) = m(T) + \sqrt{n}w(T)$. This will be the maximum distance of any input in $\cup_{B\in\B} B$ from the origin.
\end{enumerate}

% Readers are encouraged to refer to the proof roadmap in \autoref{sec:roadmap} while reading the proof.

\begin{lemma}[Hoeffding's Lemma, Lemma 2.2 in \citealp{boucheron2013concentration}]\label{lem:hoeffding}
If $Z$ is a random variable taking values in the bounded interval $[a,b]$, then $Z$ is $(\frac{b-a}{2})$-subgaussian.
\end{lemma}

\begin{lemma}[Hoeffding's Inequality, subgaussian version]
\label{lem:concentrate}
Let $X_1,\dots,X_k$ be  independent random variables with mean zero, where each $X_i$ is $\sigma_i$-subgaussian for some $\sigma_i > 0$. Then there exists an absolute constant $c > 0$ such that for any $\ep > 0$, 
\[
\Pr\left[\Bigg|\sum_{i=1}^k X_i \Bigg|  > \ep\right] \le 2\exp\left(-\frac{c\ep^2}{\sum_{i=1}^k \sigma_i^2}\right)
\]
\end{lemma}

\begin{lemma}
    \label{lem:bin-bounds}
If $x \in B \in \B$, then $|r(x,1) - \mu_B| \le L\sqrt{n}w(T)$.
\end{lemma}

\begin{proof}
Let $r^- = \inf_{x' \in B} r(x',1)$ and $r^+ = \sup_{x' \in B} r(x', 1)$. Then $r^- \le \mu_B\le r^+$ and $r^- \le r(x,1) \le r^+$. Next, for any $\ep > 0$, there exists $x^-,x^+ \in B$ such that $r(x^-,1) - \ep < r^-$ and $r(x^+,1) + \ep > r^+$ (if not, this contradicts $r^-$ and $r^+$ being the infimum and supremum). Then $r(x,1)$ and $\mu_B$ belong to the interval $[r^-,r^+]$, which is a subset of the interval $[r(x^-,1)-\ep, r(x^+,1)+\ep]$.

The maximum distance between two points in an $n$-cube with side length $w(T)$ is equal to the length of the diagonal of this $n$-cube, which is $\sqrt{n}w(T)$. Since $x^-$ and $x^+$ belong to the same $n$-cube with side length $w(T)$, we have $\norm{x^--x^+} \le \sqrt{n}w(T)$. Then by Lipschitz continuity,  $|r(x^+,1) - r(x^-,1)| = r(x^+,1) - r(x^-,1)\le L\sqrt{n}w(T)$. Therefore $r^+ - r^- \le L\sqrt{n}w(T) + 2\ep$. Since $r^- \le r(x,1) \le r^+$ and $r^- \le \mu_B \le r^+$, we have $|r(x,1) - \mu_B| \le L\sqrt{n}w(T)+2\ep$. Since this holds for all $\ep > 0$, we must have $|r(x,1) - \mu_B| \le L\sqrt{n}w(T)$.
\end{proof}

\lemPerBinConcentration*

\begin{proof}

Fix any bin $B \in \B$.

\textbf{Part 1: The reward tape construction.} We will prove equivalence between the learning process and alternative stochastic process where independence is easier to analyze. The idea is to sample upfront a list of all of the rewards we will ever need for bin $B$. Then, whenever $x_t \in B$ and $y_t = 1$ occurs, the agent simply observes the next unused reward from our list. This is essentially the ``reward tape'' argument used in Section 1.3.1 of \citet{slivkins2019introduction}. 

To formalize this, let $K$ be the maximum possible (over all realizations of randomness) number of time steps the agent commits on an input in $B$. Before $t=1$, sample $\eta^1_B,\dots,\eta^K_B$ i.i.d. from the same distribution as each $\eta_t$ and sample $x^1_B,\dots,x^K_B$ i.i.d. from the distribution of $\nu$ restricted to $B$, i.e., $x^1_B,\dots,x^K_B \sim \nu(x \mid x \in B)$. For each $k \in [K]$, let $q^k_B = r(x^k_B,1) + \eta^k_B$. For each $t \in [T]$, if $x_t \in B$ and $y_t = 1$, the agent observes $q^{k_B(t)}_B$ instead of the standard reward $r_t = r(x_t,1) + \eta_t$. If $y_t = 0$, the agent observes reward $0$ as usual.

% Note that the sample space of $r_B^k$ is the set of runs where $k_B(T) \ge k$ occurs, not the set of all runs.

We claim this is equivalent to the normal learning protocol in the following sense. For $k \in [K]$, let $r_B^k$ be the random variable of the $k$th reward observed from committing on an input in bin $B$ in the normal learning protocol. For brevity, for any $k \in [K]$, let $r_B^{1:k}$ and $q_B^{1:k}$ denote $r_B^1,\dots,r_B^k$ and $q_B^1,\dots,q_B^k$ respectively. Let $\D(X)$ denote the distribution of a random variable $X$. We claim that $\D(r_B^{1:k}) = \D(q^{1:k}_B)$. We proceed by induction from $k = 0$ to $K$ (for completeness, define $r_B^0 = q^0_B$ as any constant). The base case at $k=0$ is trivial, so suppose that $\D(r_B^{1:{k-1}}) = \D(q^{1:k-1}_B)$ for some $i \in [k]$. Let $t_k = \min \{t\in[T]: k_B(t) = k\}$ be the time step on which the $k$th reward in bin $B$ is observed. Then by definition we have $
\D(r_B^k \mid r_B^{1:k-1}) = \D(r(x_t,1) + \eta_t \mid t = t_k, r_B^{1:k-1})$ for any $t \in [T]$.\looseness=-1

We claim that conditioned on $x_t \in B$, $r(x_t,1) + \eta_t$ is independent of $t=t_k$ and $r_B^{1:k-1}$. Note that the event $\{t = t_k\}$ is equivalent to $\{x_t \in B \text{ and } y_t = 1 \text{ and } k_B(t-1) = k-1\}$. The key property is that the algorithm's behavior never distinguishes between inputs in the same bin. Formally, conditioned on $x_t \in B$, $y_t$ and $x_t$ are independent. This remains true when also conditioning on the events of prior time steps (in particular, $k_B(t-1)$ and $r_B^{1:k-1}$). Hence
\begin{align*}
\D(r_B^k \mid r_B^{1:k-1}) =&\ \D(r(x_t,1)+\eta_t \mid x_t \in B, y_t = 1, r_B^{1:k-1}, k_B(t-1) = k-1)\\
=&\ \D(r(x_t,1)+\eta_t \mid x_t \in B, r_B^{1:k-1}, k_B(t-1) = k-1)\\
=&\ \D(r(x_t,1)+\eta_t \mid x_t \in B)
\end{align*}
with the last time step because $x_t$ and $\eta_t$ are independent of the events of prior time steps.

By construction, $\D(x_B^k) = \D(x_t \mid x_t \in B)$ and $\D(\eta_B^k) = \D(\eta_t)$. Since $x_B^k$ and $\eta_B^k$ are independent (including when conditioned on $x_t \in B$), we have $\D(x_B^k,\eta_B^k) = \D(x_t,\eta_t \mid x_t \in B)$. Applying the transformation $x,\eta\mapsto r(x,1) + \eta$, we get $\D(r(x_B^k, 1) + \eta_B^k) =  \D(r(x_t,1) + \eta_t \mid x_t \in B)$. We showed above that $\D(r_B^k \mid r_B^{1:k-1}) = \D(r(x_t,1) + \eta_t \mid x_t \in B)$ and we have $\D(q_B^k) = \D(r(x_B^k,1) + \eta_B^k)$ by definition. Thus $\D(q_B^k) = \D(r_B^k \mid r_B^{1:k-1})$.

Since $x_B^{1:k},\eta_B^{1:k}$ are all independent, $q_B^k$ is independent of $q_B^{1:k-1}$. Thus  $\D(q_B^k \mid q_B^{1:k-1}) = \D(q_B^k) = \D(r_B^k \mid r_B^{1:k-1})$. Combining with this the inductive hypothesis $\D(q_B^{1:k-1}) = \D(r^{1:k-1}_B)$ gives us $\D(q_B^{1:k}) = \D(r^{1:k}_B)$, as needed. This completes the induction and shows that $\D(q_B^{1:k}) = \D(r^{1:k}_B)$ for any $k \in [K]$.

Recall that $\hmub(k)$ is the mean of the first $k$ rewards observed from committing on an input in $B$: $\hmub(k) = \frac{1}{k}\sum_{i=1}^k r_B^i$. Therefore for all $k \in [K]$, $\D(k\hmub(k)) = \D( \sum_{i=1}^k q^i_B)$.

\textbf{Part 2: Applying Lemma~\ref{lem:concentrate}.} Now fix some $k \in [K]$ and $B \in \B$. For each $i \in [k]$, let $Z_i = r(x^i_B, 1) - \mu_B$. By construction $\E[r(x^i_B, 1)] = \E_{x\sim \nu}[r(x,1) \mid x \in B] = \mu_B$, so $\E[Z_i] = 0$. Since $x^i_B \in B \in \B$, Lemma~\ref{lem:bin-bounds} implies that $|r(x^i_B,1) - \mu_B| \le L\sqrt{n}w(T)$. Thus $Z_i$ belongs to an interval of length $2L\sqrt{n}w(T)$, so $Z_i$ is $(L\sqrt{n}w(T))$-subgaussian. Also, $\eta^i_B$ is $\sigma$-subgaussian with mean 0, since $\eta^i_B$ has the same distribution as any $\eta_t$. Furthermore, our construction ensures that $Z_1,\dots,Z_k, \eta^1_B,\dots,\eta^k_B$ are all independent. Therefore by Lemma~\ref{lem:concentrate}, for any $\ep>0$,
\begin{align*}
\Pr\left[\Bigg|\sum_{i=1}^k (r(x^i_B, 1) - \mu_B + \eta^i_B) \Bigg|> \ep\right] =&\ 
\Pr\left[\Bigg|\sum_{i=1}^k Z_i + \sum_{i=1}^k \eta^i_B \Bigg|
> \ep\right]\\
\le&\ 2\exp\left(-\frac{c\ep^2}{\sum_{i=1}^k(L\sqrt{n}w(T))^2 + \sum_{i=1}^k\sigma^2}\right)\\
=&\ 2\exp\left(-\frac{c\ep^2}{k\sigma_w^2}\right)
\end{align*}
\textbf{Part 3: Converting the inequality to $\hmub(k_B(t))$.} By Part 1, $\D(k\hmub(k)) = \D(\sum_{i=1}^k q^k_B)$ for any $k \in [K]$. Since $\mu_B$ is a constant, we have
\[
\D(k\hmub(k) - k\mu_B) =\D\left(\sum_{i=1}^k q^i_B - k\mu_B\right) = \D\left(\sum_{i=1}^k (r(x^i_B, 1) - \mu_B + \eta^i_B)\right)
\]
Therefore
\[
\Pr\big[|k\hmub(k) -k\mu_B| > \ep\big] = \Pr\left[\Bigg|\sum_{i=1}^k (r(x^i_B, 1) - \mu_B + \eta^i_B) \Bigg|> \ep\right] \le 2\exp\left(-\frac{c\ep^2}{k\sigma_w^2}\right)
\]
Now set $\ep=k\gamma(k)=\sqrt{kc^{-1}\sigma_w^2\ln(2T^4)}$ to get
\begin{align*}
\Pr\left[|\hmub(k) - \mu_B| > \gamma(k)\right] =&\ \Pr\left[k|\hmub(k) - \mu_B| > \sqrt{kc^{-1}\sigma_w^2 \ln(2T^4)}\right]\\
=&\ \Pr\left[\Bigg|\sum_{i=1}^k Z_i + \sum_{i=1}^k \eta^i_B \Bigg| > \sqrt{kc^{-1}\sigma_w^2 \ln(2T^4)}\right]\\
\le&\ 2\exp(-\ln(2 T^4))\\
=&\ 2\exp\left(\ln\left(\frac{1}{2 T^4}\right)\right)\\
=&\ T^{-4}
\end{align*}
The last step is to convert $\hmub(k) $ to $\hmub(k_B(t))$. Suppose $|\hmub(k_B(t)) - \mu_B| > \gamma(k_B(t))$ for some $B \in \B,t\in[T]$ with $k_B(t) > 0$. Then $\exists k \in [T]$ such that $k_B(t) = k$ and $|\hmub(k) - \mu_B| > \gamma(k)$. Thus by the union bound,
\begin{align*}
\Pr\left[|\hmub(k_B(t)) - \mu_B| > \gamma(k_B(t))\right] =&\ \Pr\big[\exists k \in [T]: k_B(t) = k \textnormal{ and } |\hmub(k) - \mu_B| > \gamma(k)\big]\\
\le&\ \sum_{k=1}^T \Pr\big[k_B(t) = k \textnormal{ and } |\hmub(k) - \mu_B| > \gamma(k)\big]\\
\le&\ \sum_{k=1}^T \Pr\big[|\hmub(k) - \mu_B| > \gamma(k)\big]\\
\le&\ \sum_{k=1}^T T^{-4}\\
=&\ T^{-3}
\end{align*}
\end{proof}

% \begin{lemma}[Uniform concentration bound]
% \label{lem:uniform_concentration}
% We have $\Pr(\neg \mathcal{G})\le T^{-2}$.
% \end{lemma}

\lemUniformConcentration*

\begin{proof}
Let $J$ be the the number of bins that ever receive at least one commit, i.e., $J = |\{B \in \B: \exists t\in[T] \text{ s.t. } x_t \in B, y_t = 1\}|$. For each $j \in [J]$, let $B_j$ be the $j$th bin to receive a commit.
\begin{align*}
&\Pr[\neg \G ] =\ \E[\Pr[\neg \G \mid J,B_1,\dots,B_J]] && \textnormal{(Law of total expectation)}\\
&=\ \E\left[\Pr\left[\bigcup_{j=1}^J\: \bigcup_{t: k_{B_j}(t) > 0} \{|\hat{\mu}_{B_j}(k_{B_j}(t)) - \mu_{B_j}| > \gamma(k_{B_j}(t))\} \right]\ \Big|\ J,B_1,\dots,B_J\right] && \textnormal{(Direct negation)}\\
&\le\ \E\left[\sum_{j=1}^J\: \sum_{t: k_{B_j}(t) > 0} \Pr[|\hat{\mu}_{B_j}(k_{B_j}(t)) - \mu_{B_j}| > \gamma(k_{B_j}(t))] \ \Big|\ J,B_1,\dots,B_J\right] && \textnormal{(Union bound)}\\
&\le\ \E\left[\sum_{j=1}^J\: \sum_{t: k_{B_j}(t) > 0} T^{-3} \ \Big|\ J,B_1,\dots,B_J\right] && \textnormal{(Lemma~\ref{lem:per-bin-concentration})}\\
&\le\ \E\left[T^{-1}\ \Big|\ J,B_1,\dots,B_J\right] && \textnormal{$(J \in [T])$}\\
&\le\ T^{-1} && \textnormal{(Expectation of a constant)}
\end{align*}
as required.
\end{proof}

\lemSamplesForCert*

\begin{proof}
Note that $\mu_B< - L\sqrt{n}w(T)$ implies that the denominator is well-defined. By assumption on $k_B(t)$,
\begin{align*}
\gamma(k_B(t)) =&\ \sqrt{\frac{c^{-1}\sigma_w^2\ln(2T^4)}{k_B(t)}}\\ 
<&\ \sqrt{\frac{(\mu_B + L\sqrt{n}w(T))^2}{4}}\\
=&\ \frac{|\mu_B + L\sqrt{n}w(T)|}{2}\\
=&\ -\frac{\mu_B + L\sqrt{n}w(T)}{2}
\end{align*}
By definition of $\G$, we have $- \gamma(k_B(t))\le \hmub(k_B(t)) - \mu_B \le \gamma(k_B(t))$, so
\begin{align*}
\hmub(k_B(t)) + \gamma(k_B(t)) + L\sqrt{n}w(T) \le&\ \mu_B + 2\gamma(k_B(t)) + L\sqrt{n}w(T)\\
<&\ \mu_B - (\mu_B + L\sqrt{n}w(T) + L\sqrt{n}w(T)\\
=&\ 0
\end{align*}
so $B$ is certified negative at time $t$.
\end{proof}

% \begin{lemma}
% \label{lem:ball}
% For all $x \in \cup_{B\in \B} B$, we have $\norm{x} \le R(T)$.
% \end{lemma}

\lemBall*

\begin{proof}
If $x \in B$ for some $B \in \B$, there must exist $x' \in B$ such that $\norm{x'} \le m(T)$. The maximum distance between any pair of points in an $n$-cube with side length $w(T)$ is $\sqrt{n}w(T)$. Thus by the triangle inequality, $x$ satisfies
\[
\norm{x} \le \norm{x'} + \norm{x-x'} \le m(T) + \sqrt{n}w(T) = R(T)
\]
as required.
\end{proof}

% \begin{lemma}
% \label{lem:regret-per-bin}
% If $\G$ holds, then for any $B \in \B$ with $k_B(T) \ge 1$ and $\mu_B < -(2L\sqrt{n}+1)w(T)$, 
% \[
% \sum_{t: x_t \in B, y_t = 1} \Delta_t \le 2LR(T) + \frac{8c^{-1}\sigma_w^2\ln(2T^4)}{w(T)}
% \]
% \end{lemma}

\lemRegretPerBin*

\begin{proof}
Since $\mu_B < - (2\sqrt{n} L+1)w(T) < - L\sqrt{n}w(T)$ and $\G$ holds, Lemma~\ref{lem:samples_for_cert} implies that $B$ is certified negative on the first time step $t$ such that $k_B(t) > \frac{4c^{-1}\sigma_w^2\ln(2T^4)}{(\mu_B + L\sqrt{n}w(T))^2}$. Once $B$ is certified negative, we never again commit in $B$, so the number of commits in $B$ is  $|\{t \in [T]: x_t \in B, y_t = 1\}| \le \lceil \frac{4c^{-1}\sigma_w^2\ln(2T^4)}{(\mu_B + L\sqrt{n}w(T))^2}\rceil \le 1 + \frac{4c^{-1}\sigma_w^2\ln(2T^4)}{(|\mu_B| - L\sqrt{n}w(T))^2}$. Since $|\mu_B| \ge (2L\sqrt{n}+1)w(T) \ge 2L\sqrt{n}w(T)$, we have
\begin{align*}
|\mu_B| = \frac{|\mu_B|}{2} + \frac{|\mu_B|}{2}
\ge \frac{|\mu_B|}{2} + L\sqrt{n}w(T)
\end{align*}
so $|\mu_B| - L\sqrt{n}w(T) \ge |\mu_B/2|$. Therefore $(|\mu_B| - L\sqrt{n}w(T))^2 \ge \mu_B^2 /4$, so $|\{t \in [T]: x_t \in B, y_t = 1\}| \le 1 +\frac{16 c^{-1}\sigma_w^2\ln(2T^4)}{\mu_B^2}$.

For any $t \in [T]$ such that $y_t = 1$, either $y_t = 1$ is optimal or $y_t = 0$ is optimal. In the former case,  the single-step regret $\Delta_t$ is 0. In the latter case, $\Delta_t = -r(x_t,1)$. If $x_t \in B$, Lemma~\ref{lem:bin-bounds} implies that $r(x_t,1) \ge \mu_B - L\sqrt{n}w(T)$. Since $\mu_B < -L\sqrt{n}w(T)$, we have $r(x_t,1) \ge 2\mu_B$. Hence
\begin{align*}
\sum_{t: x_t \in B, y_t = 1} \Delta_t \le&\ \sum_{t: x_t \in B, y_t = 1}  (-2\mu_B)\\
=&\ |\{t \in [T]: x_t \in B, y_t = 1\}| \cdot 2|\mu_B|\\
\le&\ \left(1+\frac{16c^{-1}\sigma_w^2\ln(2T^4)}{\mu_B^2}\right) \cdot 2|\mu_B|\\
=&\ 2|\mu_B| + \frac{32c^{-1}\sigma_w^2\ln(2T^4)}{|\mu_B|}\\
\le&\ 2|\mu_B|+\frac{32c^{-1}\sigma_w^2\ln(2T^4)}{(2\sqrt{n} L+1)w(T)}\\
\le&\ 2|\mu_B|+\frac{32c^{-1}\sigma_w^2\ln(2T^4)}{w(T)}
\end{align*}
with the last step due to $2\sqrt{n}L+1 \ge 1$. By Lemma~\ref{lem:ball}, any $x \in B$ satisfies $\norm{x} \le R(T)$. Thus by Lipschitz continuity, $r(x,1) \ge r(0,1) - LR(T) > -LR(T)$ for all $x \in B$.  Thus$\mu_B = \E_{x\sim \nu}[r(x,1) \mid x \in B] \ge \E_{x\sim \nu}[-LR(T)] = -LR(T)$, so
\[
\sum_{t: x_t \in B, y_t = 1} \Delta_t \le 2LR(T)+ \frac{32 c^{-1}\sigma_w^2\ln(2T^4)}{w(T)}
\]
as required.
\end{proof}

% \begin{lemma}
% \label{lem:total-commit-regret}
% If $\G$ holds, then
% \[
% \sum_{t: y_t = 1} \Delta_t \le \frac{v_1 R(T)^n}{w(T)^n} \left(2LR(T) + \frac{8c^{-1}\sigma_w^2\ln(2T^4)}{w(T)}\right) + (3L\sqrt{n} + 1)w(T)T
% \]
% \end{lemma}
\lemTotalCommit*

\begin{proof}
For each $t \in [T]$, let $B(t)$ denote the bin to which $x_t$ belongs. Note that $B(t)$ is unique by the construction of $\mathcal{H}$ in \Cref{alg:main} and note that $B(t)$ may not be in $\B$.  Partition the time steps with commits into $S_1 = \{t \in [T]: y_t = 1 \text{ and } \mu_{B(t)} < -(2L\sqrt{n}+1)w(T)\}$ and $S_2 = \{t \in [T]: y_t = 1 \text{ and } \mu_{B(t)} \ge -(2L\sqrt{n}+1)w(T)\}$. Let $\B_1 = \{B \in \B: \exists t \in S_1 \text{ s.t. } B(t) = B\}$ be the set of bins associated with time steps in $S_1$. Then we can write
\begin{align*}
\sum_{t \in S_1} \Delta_t =&\ \sum_{B \in \B_1} \sum_{t \in S_1: B(t) = B} \Delta_t\\
=&\ \sum_{B \in \B_1} \sum_{t: x_t \in B, y_t = 1} \Delta_t
\end{align*}
Then by Lemma~\ref{lem:regret-per-bin}, 
\begin{align*}
\sum_{t \in S_1} \Delta_t \le&\ \sum_{B \in \B_1} \left(2LR(T) + \frac{32 c^{-1}\sigma_w^2\ln(2T^4)}{w(T)}\right)\\
\le&\ |\B_1| \left(2LR(T) + \frac{32 c^{-1}\sigma_w^2\ln(2T^4)}{w(T)}\right)\\
\le&\ |\B| \left(2LR(T) + \frac{32 c^{-1}\sigma_w^2\ln(2T^4)}{w(T)}\right)
\end{align*}
By Lemma~\ref{lem:ball}, every $x \in \cup_{B\in\B} B$ satisfies $\norm{x} \le R(T)$. Thus $\cup_{B\in\B} B$ is fully contained within an $n$-ball of radius $r$. The volume of such a ball is $v_1 R(T)^n$. Each bin in $\B$ is an $n$-cube with side length $w(T)$ and thus has volume $w(T)^n$. Furthermore, the bins in $B$ have no volume overlap by construction, so the total volume of bins in $B$ is $w(T)^n |\B|$. Then $w(T)^n|\B| \le v_1 R(T)^n$. Therefore
\[
\sum_{t \in S_1} \Delta_t \le \frac{v_1 R(T)^n}{w(T)^n} \left(2LR(T) + \frac{32 c^{-1}\sigma_w^2\ln(2T^4)}{w(T)}\right)
\]
Now consider any $t \in S_2$. By definition, $x_t \in B(t) \in \B$, so Lemma~\ref{lem:bin-bounds} implies that $r(x_t,1) \ge \mu_{B(t)} - L\sqrt{n}w(T)$. Since $\mu_{B(t)} \ge -(2L\sqrt{n}+1)w(T)$ by construction of $S_2$, we have $r(x_t,1) \ge -(3L\sqrt{n} + 1)w(T)$. Therefore
\begin{align*}
\sum_{t \in S_2} \left(\max_{y^* \in \{0,1\}} r(x_t,y^*) - r(x_t,1)\right) \le&\ \sum_{t \in S_2} (3L\sqrt{n} + 1)w(T)\\
=&\ |S_2| (3L\sqrt{n} + 1)w(T)\\
\le&\ (3L\sqrt{n} + 1)w(T)T
\end{align*}
Putting it all together,
\begin{align*}
\sum_{t: y_t = 1} \Delta_t =&\ \sum_{t \in S_1} \left(\max_{y^* \in \{0,1\}} r(x_t,y^*) - r(x_t,1)\right) + \sum_{t \in S_2} \left(\max_{y^* \in \{0,1\}} r(x_t,y^*) - r(x_t,1)\right)\\
\le&\ \frac{v_1 R(T)^n}{w(T)^n} \left(2LR(T) + \frac{32 c^{-1}\sigma_w^2\ln(2T^4)}{w(T)}\right) + (3L\sqrt{n} + 1)w(T)T
\end{align*}
as required.
\end{proof}

\thmMain*

\begin{proof}
First assume $\G$ holds. Let $S_3 = \{t \in [T]: y_t = 1 \text{ and } r(x_t,1) < r(x_t, 0)\}$ be the time steps where we committed but we should have abstained, and let $S_4 = \{t \in [T]: y_t = 0 \text{ and } r(x_t,0) < r(x_t, 1)\}$ be the time steps where we should have committed but we abstained. Lemma~\ref{lem:total-commit-regret} bounds the regret of time steps in $S_3$. Since we always commit whenever $x_t \in B$ for some $B \in \B$, $S_4$ can only occur when $x_t \not\in B$ for all $B \in \B$. By construction, any such $x_t$ satisfies $\norm{x_t} > m(T)$ (otherwise the bin containing $x_t$ would be in $\B$). Also, $r(x_t,1) - r(x_t, 0) \le 1$ by assumption. Hence
\begin{align*}
\RegT =&\ \sum_{t =1}^T \Delta_t\\
=&\ \sum_{t \in S_3} \Delta_t + \sum_{t \in S_4} \Delta_t\\
\le&\ \frac{v_1R(T)^n}{w(T)^n} \left(2LR(T) + \frac{32 c^{-1}\sigma_w^2\ln(2T^4)}{w(T)}\right) + (3L\sqrt{n} + 1)w(T)T + \sum_{t=1}^T \bfone(\norm{x_t} > m(T) )\\
=&\ \frac{2Lv_1R(T)^{n+1}}{w(T)^n}  + \frac{32v_1 c^{-1}\sigma_w^2 R(T)^n\ln(2T^4)}{w(T)^{n+1}} + (3L\sqrt{n} + 1)w(T)T + \sum_{t=1}^T \bfone(\norm{x_t} > m(T) )
\end{align*}
Therefore
\begin{align*}
\E[\RegT \mid \G] \le&\ \frac{2Lv_1R(T)^{n+1}}{w(T)^n}  + \frac{32 v_1 c^{-1}\sigma_w^2 R(T)^n\ln(2T^4)}{w(T)^{n+1}} + (3L\sqrt{n} + 1)w(T)T + \sum_{t=1}^T \Pr[\norm{x_t} > m(T) ]\\
=&\ \frac{2Lv_1R(T)^{n+1}}{w(T)^n}  + \frac{32 v_1 c^{-1}\sigma_w^2 R(T)^n\ln(2T^4)}{w(T)^{n+1}} + (3L\sqrt{n} + 1)w(T)T + \sum_{t=1}^T \bar{\nu}(m(T))\\
=&\ \frac{2Lv_1R(T)^{n+1}}{w(T)^n}  + \frac{32 v_1 c^{-1}\sigma_w^2 R(T)^n\ln(2T^4)}{w(T)^{n+1}} + (3L\sqrt{n} + 1)w(T)T + T\bar{\nu}(m(T) )
\end{align*}
Now suppose $\G$ does not hold. Consider an arbitrary $t \in [T]$. If $y_t = 0$, then the regret at time $t$ is at most 1. If $y_t = 1$, we still have $\norm{x_t} \le R(T)$, so by Lipschitz continuity, $r(x_t, 1) \ge -LR(T)$. Therefore $\E[\RegT \mid \neg \G] \le T + LR(T)T$. Lemma~\ref{lem:uniform_concentration} implies that $\Pr[\neg \G] \le 1/T$, so by the law of expectation,
\begin{align*}
\E[\RegT] =&\ \Pr[\neg \G] \E[\RegT \mid \neg \G] + \Pr[\G] \E[\RegT \mid \G]\\
\le&\ \frac{1}{T} \cdot (T+ L R(T) T) + \frac{2Lv_1R(T)^{n+1}}{w(T)^n}  + \frac{32 v_1 c^{-1}\sigma_w^2 R(T)^n\ln(2T^4)}{w(T)^{n+1}} + (3L\sqrt{n} + 1)w(T)T + T\bar{\nu}(m(T))\\
\in&\ O\left(1+ L R(T) +\frac{2Lv_1R(T)^{n+1}}{w(T)^n}  + \frac{32 v_1 c^{-1}\sigma_w^2 R(T)^n\ln(2T^4)}{w(T)^{n+1}} + L\sqrt{n}w(T)T + T\bar{\nu}(m(T))\right)
% =&\ O\left(\frac{L(m(T)+\sqrt{n}w(T))}{T} +\frac{(L +\sigma^2_w)(m(T)+\sqrt{n}w(T))^{n+1}  \ln T}{w(T)^{n+1}} + L\sqrt{n}w(T)T + T\bar{\nu}(m(T))\right)
\end{align*}
We now plug in $w(T) = T^\frac{-1}{n+2}$ and $m(T) = \ln T$. Since $\lim_{T\to\infty} w(T) = 0$, we have $\sigma_w^2 = nL^2w(T)^2 + \sigma^2 \in O(\sigma^2)$. Similarly, for any $k \ge 0$, $R(T)^k =(\ln(T) + \sqrt{n}T^\frac{-1}{n+2})^k \in O((\ln T)^k)$. Thus
\begin{align*}
\E[\RegT] \in&\ O\left(\frac{L \ln T}{T} +\frac{L (\ln T)^{n+1}}{T^\frac{-n}{n+2}} + \frac{\sigma^2 (\ln T)^{n+1}}{T^\frac{-n-1}{n+2}} + LT^\frac{-1}{n+2}T + T\bar{\nu}(\ln T)\right)\\
=&\ O\left((L +\sigma^2)  T^\frac{n+1}{n+2} (\ln T)^{n+1} + T\bar{\nu}(\ln T )\right)
\end{align*}
as required.
\end{proof}

\newpage
\section{Simulations}\label{app:sim}

In this section, we provide some basic simulations to illustrate the performance of our algorithm. Rigorous experimental validation of our algorithm is beyond the scope of the paper, but we hope that these simulations help provide intuition for the behavior of our algorithm. Our experimental setup is as follows:
\begin{enumerate}[
    topsep=0.5ex,
    partopsep=0pt,
    parsep=0pt,
    itemsep=0.5ex]
    \item \emph{Input space.} We use synthetically generated i.i.d. inputs in $\bbr$.
    \item \emph{Input distribution.} We consider two different distributions $\nu$: a Gaussian distribution with $\sigma^2 = 5$ and a Cauchy distribution with $\gamma=1$, both centered at 0. The former is light-tailed, meaning that even reckless exploration is typically fine since very OOD inputs are unlikely. The latter is heavy-tailed, which is where OOD inputs are relatively common and so caution is crucial.
    (which is heavy-tailed).
    \item \emph{Reward function.} We consider the two extremes of possible reward functions: $r(x,1) = 1$ and $r(x,1) = 1 - |x|$. For the former, the task policy generalizes perfectly and the agent should always commit. For the latter, the task policy generalizes poorly and the agent should abstain except for a small region around the origin.
    \item \emph{Noise distribution.} In all cases, the noise variables $\eta_1,\dots,\eta_T$ are sampled from a Gaussian distribution with $\sigma^2 = 1$ centered at 0.
    \item \emph{Time horizons.} We tested time horizons $T$ up to 500,000.
    \item \emph{Algorithms.} In addition to our algorith, we tested three baselines: (1) always commit, (2) always abstain, and (3) \textsc{Lipschitz UCB1}, a standard non-cautious algorithm for Lipschitz bandits. Specifically, \textsc{Lipschitz UCB1} is the uniform-discretization algorithm described in Section 8.2 of \citet{slivkins2019introduction}, instantiated with UCB1 (as suggested in Section 8.1 of the same book). We initially planned to also include a risk-sensitive bandit algorithm but could not find one that would admit an apples-to-apples comparison (see \autoref{sec:related} for details).
\end{enumerate}

Our code can be found at \url{https://github.com/bplaut/abstention-bandits-simulations}. The simulations were not computationally intensive and were run on a personal laptop.

\autoref{fig:simulations} shows the simulation results. For readability, curves with very high regret are truncated. Each algorithm was run 10 times with 10 different seeds. The solid lines are the mean of those runs, and the shaded region indicates the maximum and minimum values across runs. In many cases, the shaded region is not visible because it is so small relative to the scale of the plots. 

The combination of input distribution and reward function results in four concrete scenarios, each corresponding to one of the panels in \autoref{fig:simulations}. The top row corresponds to $r(x,1) = 1$ and the bottom row corresponds to $r(x,1) = 1-|x|$. The left and right columns correspond to Gaussian inputs and Cauchy inputs respectively.

In the top row, $r(x,1) = 1$ means that the optimal strategy is to always commit, meaning that there is no benefit to caution. As a result, \textsc{Lipschitz UCB1} performs well for both Gaussian or Cauchy inputs. Our algorithm also performs well (as it does in all cases). As expected, ``Always abstain'' performs pessimally.

The bottom row is where things get interesting. The reward function $r(x,1) = 1-|x|$ implies that there is a small area of in-distribution inputs near the origin where the agent should commit, but otherwise the agent should abstain. The agent needs to learn the boundaries of this region. ``Always commit'' performs terribly here for both input distributions. ``Always abstain'' performs better, but still exhibits linear regret due to never learning to commit near the origin. In the bottom row, committing on far OOD inputs can have catastrophic costs, so caution becomes relevant. However, for a light-tailed input distribution (bottom left), far OOD inputs are rare enough that caution is not necessary; hence the success of \textsc{Lipschitz UCB1}. The crucial scenario is the bottom right: committing on far OOD inputs is catastrophic \emph{and} the agent actually encounters such inputs. In this scenario, \textsc{Lipschitz UCB1} is abysmal due to its reckless exploration. In contrast, our algorithm explores cautiously and successfully learns where to commit without ever risking catastrophe.
\newcommand\figscale{.49}
\begin{figure}
    \centering
    \includegraphics[width=\figscale\linewidth]{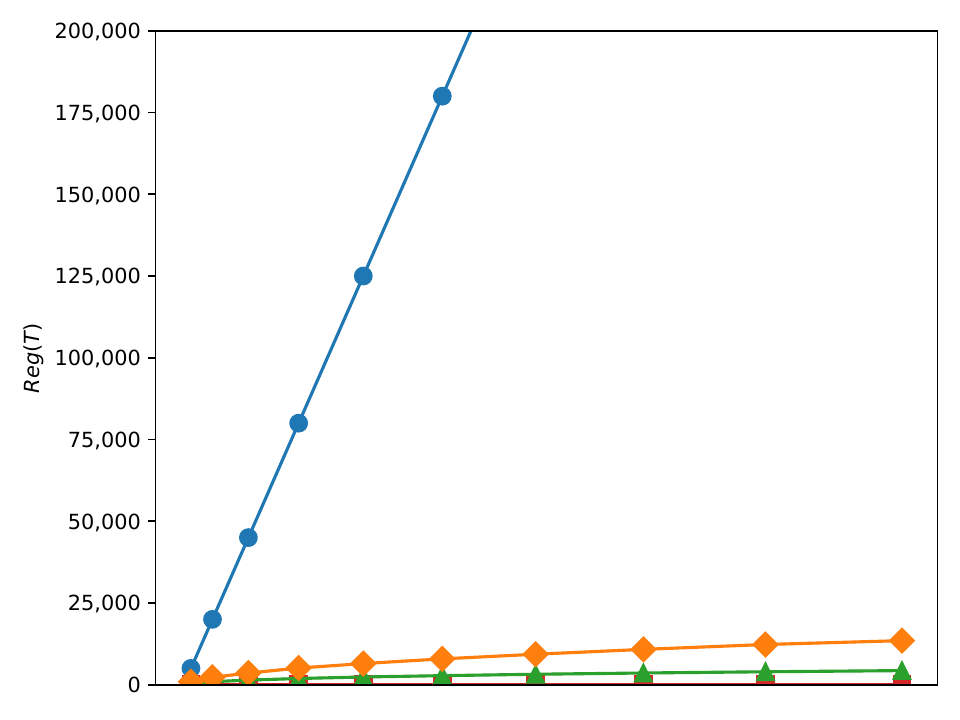}
    \includegraphics[width=\figscale\linewidth]{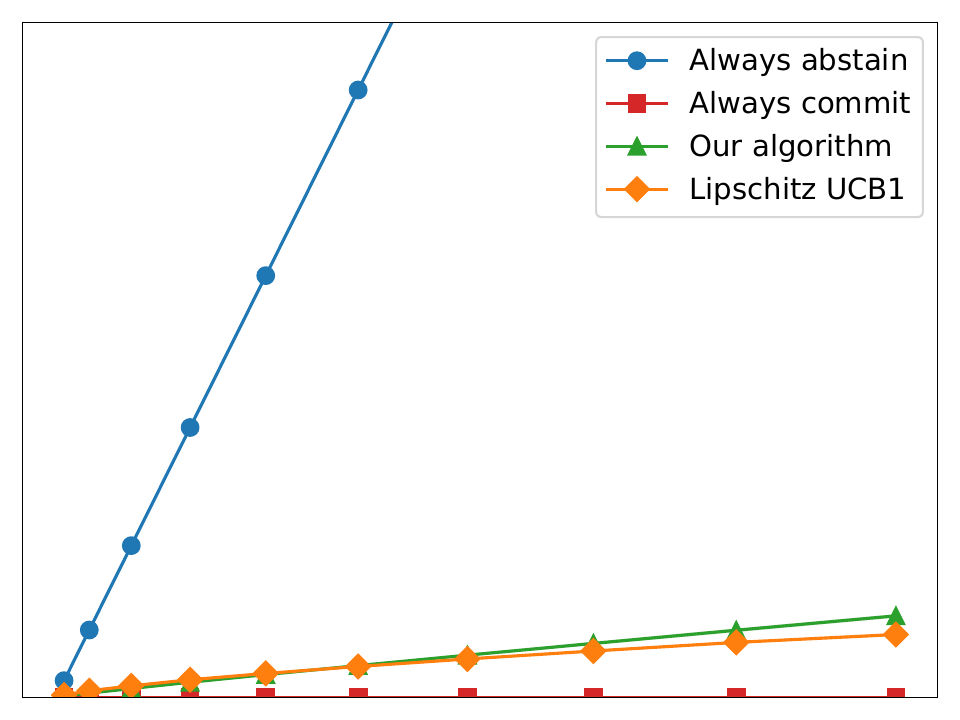}
    \includegraphics[width=\figscale\linewidth]{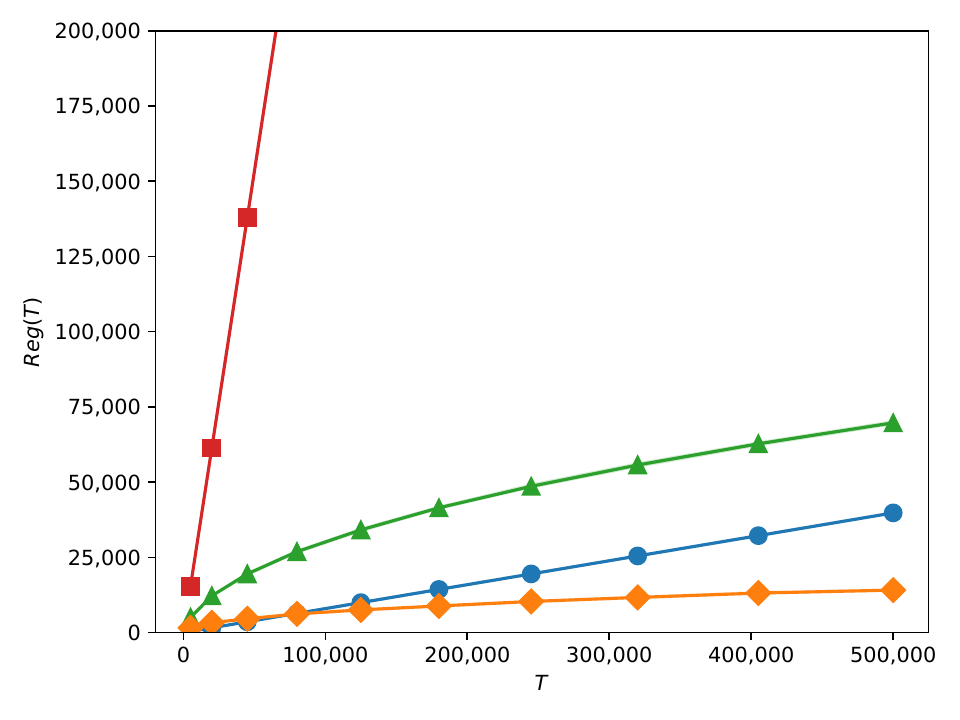}
    \includegraphics[width=\figscale\linewidth]{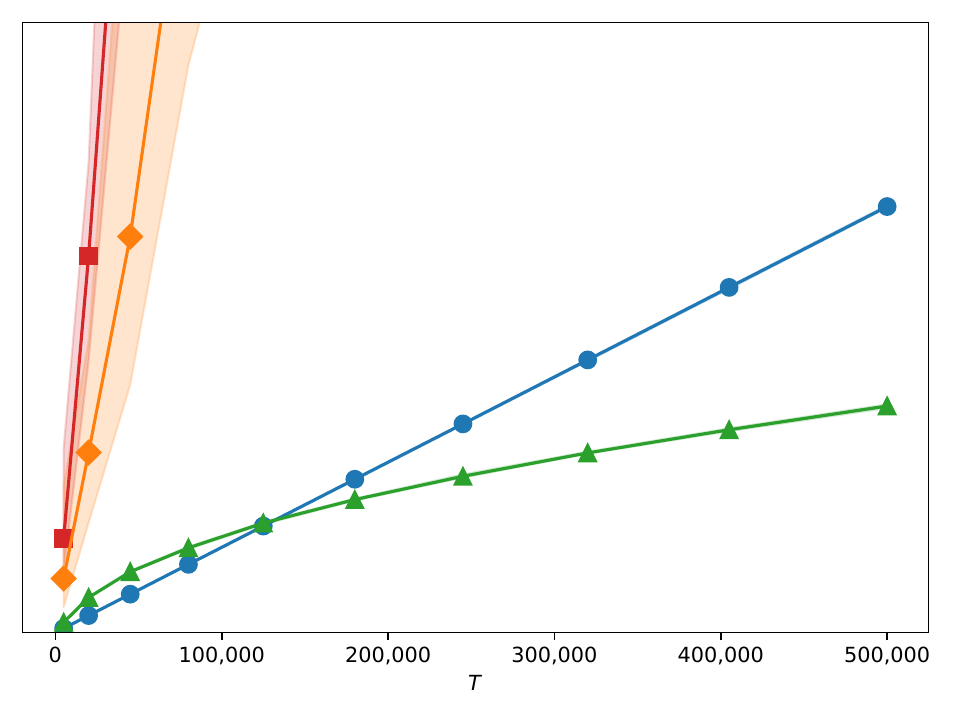}    
    \caption{Simulation results. Top row: $r(x,1) = 1$. Bottom row: $r(x,1) = 1- |x|$. Left column: inputs are i.i.d. from a Gaussian distribution. Right column: inputs are i.i.d. from a Cauchy distribution. Each curve is the mean of 10 runs with different seeds; the shaded region indicates the minimum and maximum of those runs.}
    \label{fig:simulations}
\end{figure}

We can gain further insight into the behavior of algorithm by examining the number of abstentions and certified negative bins over time. \autoref{fig:behavior} shows how these quantities evolve over time for the $r(x,1) = 1- |x|$ setting. (These quantities are less interesting for $r(x,1) = 1$, since no bins will ever be certified negative.) Unlike \autoref{fig:simulations} which compares a variety of values of $T$, \autoref{fig:behavior} studies the single value of $T = 500,000$ and plots the progress of the algorithm across time steps $t$. The left and right plots in \autoref{fig:behavior} correspond to Gaussian inputs and Cauchy inputs, respectively. In both plots, the number of certified negative bins initially increases quickly as the algorithm explores, and then plateaus. The absention rate also stabilizes as a function of the probability mass outside of the trusted region (i.e., how often we get inputs that are too OOD to even explore).

\begin{figure}
    \centering
    \includegraphics[width=\figscale\linewidth]{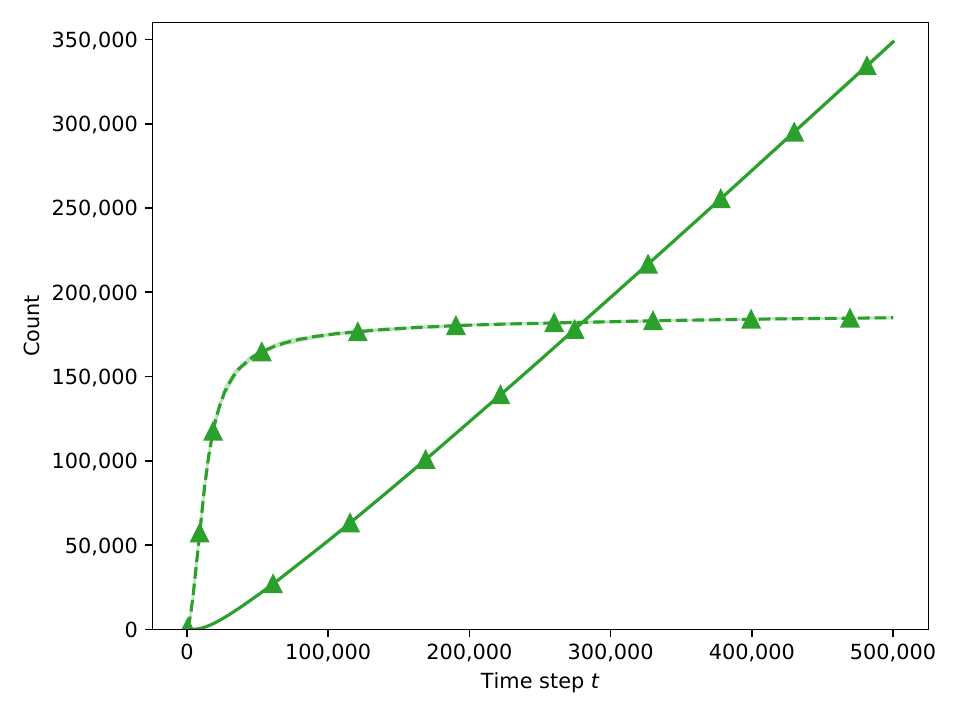}
    \includegraphics[width=\figscale\linewidth]{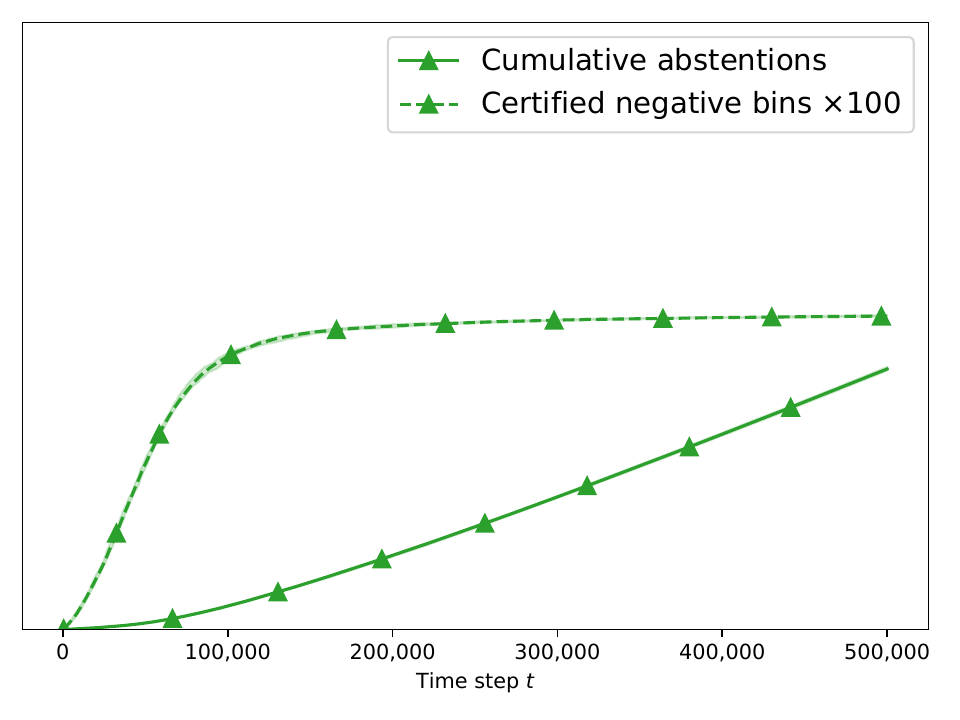}    
    \caption{Abstentions and the number of certified negative bins over time. Both plots correspond to $r(x,1) = 1-|x|$. The left plot and right plot correspond to Gaussian inputs and Cauchy inputs, respectively.}
    \label{fig:behavior}
\end{figure}

\thispagestyle{empty}

\end{document}